\documentclass[twoside,11pt]{article}
\usepackage{jmlr2ee}
\usepackage[utf8]{inputenc}
\usepackage{amssymb}
\usepackage{amsmath}
\usepackage{amsfonts}
\usepackage{stmaryrd}
\usepackage{ragged2e}
\usepackage{graphicx}
\usepackage{amsmath}
\usepackage{color}
\usepackage{dsfont}

\DeclareMathOperator*{\argmin}{argmin}
\def\n{\llbracket n \rrbracket}

\newcommand{\Sn}{\mathfrak{S}_n}

\def\I{\mathcal I}
\def\C{\mathcal C}

\def\E{\mathbb E}

\def\argmax{\mathop{\rm arg\, max}}
\def\argmin{\mathop{\rm arg\, min}}

\def\blackslug{\hbox{\hskip 1pt \vrule width 4pt height 8pt depth 1.5pt
\hskip 1pt}}
\def\qed{\quad\blackslug\lower 8.5pt\null\par}
\def\C{{\cal C}}
\def\P{\mathbf P}

\newcommand{\Prob}[1]{\mathbb{P}\left\{ #1 \right\} }

\newtheorem{property}{{\bf Property}}

\usepackage[textwidth=4.cm, textsize=tiny]{todonotes}

\usepackage{hyperref}

\makeatletter
\newcommand{\printfnsymbol}[1]{%
	\textsuperscript{\@fnsymbol{#1}}%
}
\makeatother

\title{Dimensionality Reduction and (Bucket) Ranking: \\
a Mass Transportation Approach}

\author{\name{Mastane Achab\thanks{equal contribution}} \email{mastane.achab@telecom-paristech.fr }
	\AND
	\name{Anna Korba\footnotemark[1]} \email{anna.korba@telecom-paristech.fr}
	\AND
	\name{Stephan Cl\'emen\c{c}on} \email{stephan.clemencon@telecom-paristech.fr}\\
	\addr LTCI, T\'el\'ecom ParisTech, Universit\'e Paris-Saclay, France\\
}

\begin{document}

\maketitle

\begin{abstract}
Whereas most dimensionality reduction techniques (\textit{e.g.} PCA, ICA, NMF) for multivariate data essentially rely on linear algebra to a certain extent, summarizing ranking data, viewed as realizations of a random permutation $\Sigma$ on a set of items indexed by $i\in \{1,\ldots,\; n\}$,  is a great statistical challenge, due to the absence of vector space structure for the set of permutations $\mathfrak{S}_n$. It is the goal of this article to develop an original framework for possibly reducing the number of parameters required to describe the distribution of a statistical population composed of rankings/permutations, on the premise that the collection of items under study can be partitioned into subsets/buckets, such that, with high probability, items in a certain bucket are either all ranked higher or else all ranked lower than items in another bucket. In this context, $\Sigma$'s distribution can be hopefully represented in a sparse manner by a \textit{bucket distribution}, \textit{i.e.} a bucket ordering plus the ranking distributions within each bucket. More precisely, we introduce a dedicated distortion measure, based on a mass transportation metric, in order to quantify the accuracy of such representations. The performance of buckets minimizing an empirical version of the distortion is investigated through a rate bound analysis. Complexity penalization techniques are also considered to select the shape of a bucket order with minimum expected distortion. Beyond theoretical concepts and results, numerical experiments on real ranking data are displayed in order to provide empirical evidence of the relevance of the approach promoted.
\end{abstract}


\section{Introduction}
Recommendation systems and search engines are becoming ubiquitous in modern technological tools. Operating continuously on still more content, use of such tools generate or take as input more and more data.  
The scientific challenge relies on the nature of the data feeding or being produced by such algorithms: input or/and output information generally consists of rankings/orderings, expressing \textit{preferences}. Because the number of possible rankings explodes with the number of instances, it is of crucial importance to elaborate dedicated dimensionality reduction methods in order to represent ranking data efficiently. Whatever the type of task considered (supervised, unsupervised), machine-learning algorithms generally rest upon the computation of statistical quantities such as averages or linear combinations of the observed features, representing efficiently the data.
However, summarizing ranking variability is far from straightforward and extending simple concepts such as that of an average or median in the context of preference data raises a certain number of deep mathematical and computational problems.
For instance, whereas it is always possible to define a barycentric permutation (\textit{i.e.} a consensus ranking) given a set of rankings and a metric on the symmetric group, its computation can be very challenging, as evidenced by the increasing number of contributions devoted to the ranking aggregation problem in the machine-learning literature, see \textit{e.g.} \cite{dwork2001rank}, \cite{procaccia2016optimal}, \cite{JKSO16} or \cite{JKS16} among others. Regarding dimensionality reduction, it is far from straightforward to adapt traditional techniques such as Principal Component Analysis and its numerous variants to the ranking setup, the main barrier being the absence of a vector space structure on the set of permutations. Even if one can embed permutations into the Birkhoff polytope (which is the convex hull of the set of permutation matrices, see \cite{clemenccon2010kantorovich},\cite{LindermanMena2017reparameterizing}), the coordinates of the embeddings are highly correlated, and a low-dimensional representation of the original distribution over rankings could not be interpreted in a straightforward manner. In this paper, we develop a novel framework for representing the distribution of ranking data in a simple manner, that is shown to extend, remarkably, consensus ranking in some sense. The rationale behind the approach we promote is that, in many situations encountered in practice, the set of instances may be partitioned into subsets/buckets, such that, with high probability, objects belonging to a certain bucket are either all ranked higher or else all ranked lower than objects lying in another bucket. In such a case, the ranking distribution can be described in a sparse fashion by: 1) a gross ordering structure (related to the buckets) and 2) the marginal ranking distributions associated to each bucket. Precisely, optimal representations are defined here as those associated to a bucket order minimizing a certain distortion measure we introduce, the latter being based on a mass transportation metric on the set of ranking distributions. Noticeably, this distortion measure is shown to admit a very simple closed-form expression, based on the marginal pairwise probabilities solely, when the cost of the mass transportation metric considered is the Kendall's $\tau$ distance and can be thus straightforwardly estimated. In the Kendall's $\tau$ case, we also highlight the fact that distortion minimization over bucket orders, when buckets are singletons, reduces to Kemeny consensus ranking. We establish rate bounds describing the generalization capacity of bucket order representations obtained by minimizing an empirical version of the distortion over collections of bucket orders and address model selection issues related to the choice of the bucket order size/shape. Numerical results are also displayed, providing in particular strong empirical evidence of the relevance of the notion of sparsity considered, which the dimensionality reduction technique introduced is based on.

The article is organized as follows. In section \ref{sec:background}, a few concepts and results pertaining to (Kemeny) consensus ranking are briefly recalled and the extended framework we consider for dimensionality reduction in the ranking context is described at length. Statistical results guaranteeing that optimal representations of reduced dimension can be learnt from ranking observations are established in section \ref{sec:learning}, while numerical experiments are presented in section \ref{sec:exp} for illustration purpose. Some concluding remarks are collected in section \ref{sec:concl}. Technical details are deferred to the Supplementary Material due to space limitations.

\section{Preliminaries - Background}\label{sec:background}
It this section, we introduce the main concepts and
definitions that shall be used in the subsequent analysis. The indicator function of any event $\mathcal{E}$ is denoted by $\mathbb{I}\{\mathcal{E}  \}$, the Dirac mass at any point $a$ by $\delta_a$, the cardinality of any finite subset $A$ by $\#A$.
Here and throughout, a full ranking on a set of items indexed by $\n=\{1,\; \ldots,\; n \}$ is seen as the permutation $\sigma\in \Sn$ that maps any item $i$ to its rank $\sigma(i)$. For any non empty subset $\I\subset \n$, any ranking $\sigma$ on $\n$ naturally defines a ranking on $\I$, denoted by $\Pi_{\I}(\sigma)$ (\textit{i.e.} $\forall i\in \I$, $\Pi_{\I}(\sigma)(i)=1+\sum_{j\in \I\setminus\{i\}}\mathbb{I}\{\sigma(j)<\sigma(i)\}$). If $\Sigma$ is a random permutation on $\Sn$ with distribution $P$, the distribution of $\Pi_{\I}(\Sigma)$ will be referred to as the marginal of $P$ related to the subset $\I$. In particular, for a pair of items $(i,j) \in \n$, the quantity $p_{i,j}=\mathbb{P}\{\Sigma(i)<\Sigma(j)\}$ for $\Sigma \sim P$ is referred to as the pairwise marginal of $P$ and indicates the probability that item $i$ is preferred to (ranked lower than) item $j$ (so $p_{i,j}+p_{j,i}=1$). A bucket order $\C$ (also referred to as a partial ranking in the literature) is a strict partial order defined by an ordered partition of $\n$, \textit{i.e} a sequence $(\mathcal{C}_1, \dots, \mathcal{C}_K)$ of $K\geq 1$ pairwise disjoint non empty subsets (buckets) of $\n$ such that: (1) $\cup_{k=1}^K \mathcal{C}_k=\n$, (2) $\forall(i,j)\in \n^2$, we have: $i\prec_{\C} j$ ($i$ is ranked lower than $j$ in $\C$) iff $\exists k< l$ s.t. $(i,j)\in \C_k\times \C_l$. We write $i \sim_{\C} j$ to mean that $i$ and $j$ belong to the same bucket (and cannot be compared/ordered by means of $\mathcal{C}$). 
 The items in $\C_1$ have thus the lowest ranks (i.e. they are the most preferred items), whereas those in $\C_K$ have the highest ranks. For any bucket order $\C=(\C_1, \dots, \C_K)$, its number of buckets $K$ is referred to as its \textit{size}, while its \textit{shape} is the vector $\lambda=(\#\mathcal{C}_1, \dots, \#\mathcal{C}_K)$, i.e the sequence of sizes of buckets in $\C$ (verifying $\sum_{k=1}^K \#\C_k=n$). Hence, any bucket order $\C$ of size $n$ corresponds to a full ranking/permutation $\sigma \in \Sn$, whereas the set of all items $\n$ is the unique bucket order of size $1$.

\subsection{Background on Consensus Ranking}

Given a collection of $N\geq 1$ rankings $\sigma_{1},\; \ldots,\; \sigma_{N}$, consensus ranking, also referred to as ranking aggregation, aims at finding a ranking $\sigma^*\in \Sn$ that best summarizes it. A popular way of tackling this problem, the metric-based consensus approach, consists in solving:
\begin{equation}
	\label{eq:ranking_aggregation}
	\min_{\sigma\in \mathfrak{S}_n}\sum_{s=1}^N d(\sigma,\sigma_{s}),
\end{equation}
where $d(.\,,\,.)$ is a certain metric on $\Sn$. As the set $\Sn$ is of finite cardinality, though not necessarily unique, such a barycentric permutation, called \textit{consensus/median ranking}, always exists.
In Kemeny ranking aggregation, the most widely documented version in the literature, one considers the number of pairwise disagreements as metric, namely the Kendall's $\tau$ distance, see \cite{Kemeny59}:
\begin{equation}\label{eq:Kendall_tau}
\forall (\sigma,\sigma')\in \mathfrak{S}_n^2,\;\; d_{\tau}(\sigma,\sigma')=\sum_{i<j}\mathbb{I}\{(\sigma(i)-\sigma(j)) (\sigma'(i)-\sigma'(j))<0\}.
\end{equation}
\begin{remark}
Many other distances are considered in the literature (see \textit{e.g.} Chapter 11 in \cite{Deza}). In particular, the following distances, originally introduced in the context of nonparametric hypothesis testing, are also widely used.
\begin{itemize}
\item[$\bullet$] {\bf The Spearman $\rho$ distance.} $\forall (\sigma,\sigma')\in \mathfrak{S}_n^2$,
$
d_2(\sigma,\sigma')=\left(\sum_{i=1}^n\left(\sigma(i)-\sigma'(i)
  \right)^2\right)^{1/2}$
\item[$\bullet$] {\bf The Spearman footrule distance.} $\forall (\sigma,\sigma')\in
  \mathfrak{S}_n^2$,
$d_1(\sigma,\sigma')=\sum_{i=1}^n\left\vert\sigma(i)-\sigma'(i)
\right\vert$
\item[$\bullet$] {\bf The Hamming distance.} $\forall (\sigma,\sigma')\in
\mathfrak{S}_n^2$,
$d_H(\sigma,\sigma')=\sum_{i=1}^n\mathbb{I}\{ \sigma(i)\ne \sigma'(i)\}$
\end{itemize}
\end{remark}
The problem \eqref{eq:ranking_aggregation} can be viewed as a $M$-estimation problem in the probabilistic framework stipulating that  the collection of rankings to be aggregated/summarized is composed of $N\geq 1$ independent copies $\Sigma_1,\; \ldots,\; \Sigma_N$ of a generic r.v. $\Sigma$, defined on a probability space $(\Omega,\; \mathcal{F},\; \mathbb{P})$ and drawn from an unknown probability distribution $P$ on $\mathfrak{S}_n$ (\textit{i.e.} $P(\sigma)=\mathbb{P}\{ \Sigma=\sigma \}$ for any $\sigma\in \mathfrak{S}_n$). Just like a median of a real valued r.v. $Z$ is any scalar closest to $Z$ in the $L_1$ sense, a (true) median of distribution $P$ w.r.t. a certain metric $d$ on $\mathfrak{S}_n$ is any solution of the minimization problem:
\begin{equation}\label{eq:median_pb}
\min_{\sigma \in \mathfrak{S}_n}L_P(\sigma),
\end{equation}
where $L_P(\sigma)=\mathbb{E}_{\Sigma \sim P}[d(\Sigma,\sigma)  ]
$ denotes the expected distance between any permutation $\sigma$ and $\Sigma$. In this framework, statistical ranking aggregation consists in recovering a solution $\sigma^*$ of this minimization problem, plus an estimate of this minimum $L^*_P=L_P(\sigma^*)$, as accurate as possible, based on the observations $\Sigma_1,\; \ldots,\; \Sigma_N$. A median permutation $\sigma^*$ can be interpreted as a central value for distribution $P$, while the quantity $L^*_P$ may be viewed as a dispersion measure.
Like problem \eqref{eq:ranking_aggregation}, the minimization problem \eqref{eq:median_pb} has always a solution but can be multimodal.
However, the functional $L_P(.)$ is unknown in practice, just like distribution $P$. Suppose that we would like to avoid rigid parametric assumptions on $P$ and only have access to the dataset $(\Sigma_1,\; \ldots,\; \Sigma_N )$ to find a reasonable approximant of a median. The Empirical Risk Minimization (ERM) paradigm, see \cite{Vapnik}, encourages us to substitute in \eqref{eq:median_pb} the quantity $L_P(\sigma)$ with its statistical version
\begin{equation}\label{eq:emp_risk}
\widehat{L}_N(\sigma)=\frac{1}{N}\sum_{s=1}^Nd(\Sigma_s,\sigma)=L_{\widehat{P}_N}(\sigma),
\end{equation}
where $\widehat{P}_N=(1/N)\sum_{s=1}^N\delta_{\Sigma_s}$ denotes the empirical measure.
The performance of empirical consensus rules, solutions $\widehat{\sigma}_N$ of
 $\min_{\sigma\in \mathfrak{S}_n}\widehat{L}_N(\sigma)$,
 has been investigated in \cite{CKS17}.  Precisely, rate bounds of order $O_{\mathbb{P}}(1/\sqrt{N})$ for the excess of risk $L_P(\widehat{\sigma}_N)-L^*_P$ in probability/expectation have been established and proved to be sharp in the minimax sense, when $d$ is the Kendall's $\tau$ distance. Whereas problem \eqref{eq:ranking_aggregation} is NP-hard in general (see \textit{e.g.} \cite{Hudry08}), in the Kendall's $\tau$ case, exact solutions, referred to as \textit{Kemeny medians}, can be explicitly derived when the pairwise probabilities $p_{i,j}=\mathbb{P}\{  \Sigma(i)<\Sigma(j)\}$, $1\leq i\neq j\leq n$, fulfill the following property, referred to as \textit{stochastic transitivity}.
 \begin{definition}\label{def:stoch_trans} Let $P$ be a probability distribution on $\mathfrak{S}_n$.\\
 \noindent $(i)$ Distribution $P$ is said to be (weakly) stochastically transitive iff
 $$
\forall (i,j,k)\in \n^3:\;\;  p_{i,j}\geq 1/2 \text{ and } p_{j,k}\geq 1/2 \; \Rightarrow\; p_{i,k}\geq 1/2.
 $$
 If, in addition, $p_{i,j}\neq 1/2$ for all $i<j$, one says that $P$ is strictly stochastically transitive.\\
\noindent $(ii)$ Distribution $P$ is said to be strongly stochastically transitive iff
 $$
 \forall (i,j,k)\in \n^3:\;\;  p_{i,j}\geq 1/2 \text{ and } p_{j,k}\geq 1/2 \; \Rightarrow\; p_{i,k}\geq max(p_{i,j}, p_{j,k}).
 $$ This is equivalent to the following condition (see \cite{davidson1959experimental}):
 $$
\forall (i,j)\in \n^2:\;\;  p_{i,j}\geq 1/2 \; \Rightarrow\; p_{i,k}\geq p_{j,k} \text{ for all } k\in\n\setminus\{i,j\}.
 $$
 \end{definition}
These conditions were firstly introduced in the psychology literature (\cite{fishburn1973binary}, \cite{davidson1959experimental}) and were used recently for the estimation of pairwise probabilities and ranking from pairwise comparisons (\cite{shah2015stochastically}, \cite{shah2015simple}). Examples of stochastically transitive distributions on $\mathfrak{S}_n$ are far from uncommon and include most popular parametric models such as Mallows or Bradley-Terry-Luce-Plackett models, see \textit{e.g.} \cite{Mallows57} or \cite{Plackett75}. When stochastic transitivity holds true, the set of Kemeny medians (see Theorem 5 in \cite{CKS17}) is the set
$\{\sigma\in \mathfrak{S}_n:\; (p_{i,j}-1/2)(\sigma(j)-\sigma(i ))>0 \text{ for all } i<j \text{ s.t. } p_{i,j}\neq 1/2  \}$, and
 the minimum is given by
 \begin{equation}\label{eq:inf}
 L^*_P=\sum_{i<j}\min\{p_{i,j},1-p_{i,j}  \}=\sum_{i<j}\{  1/2-\vert  p_{i,j}-1/2\vert\}.
 \end{equation}
 If a strict version of stochastic transitivity is fulfilled, we denote by $\sigma^*_P$ the Kemeny median which is unique and given by the Copeland ranking, that assigns for each $i$ its rank as:
 \begin{equation}\label{eq:sol_SST}
 \sigma^*_P(i)=1+\sum_{j\neq i}\mathbb{I}\{p_{i,j}<1/2  \} \text{ for } 1\leq i\leq n.
 \end{equation}
Assume that the underlying distribution $P$ is strictly stochastically transitive and verifies additionally a certain low-noise condition {\bf NA}$(h)$, defined for $h>0$ by:
\begin{equation}\label{eq:hyp_margin0}
 \min_{i<j}\left\vert p_{i,j}-1/2 \right\vert \ge h.
\end{equation}
This condition is checked in many situations, including most conditional parametric models (see Remark~13 in \cite{CKS17}) under simple assumptions on their parameters. It may be considered as analogous to that introduced in \cite{KB05} in binary classification, and was used to prove fast rates also in ranking, for the estimation of the matrix of pairwise probabilities (see  \cite{shah2015stochastically}) or ranking aggregation (see \cite{CKS17}). Indeed it is shown in \cite{CKS17} that under condition \eqref{eq:hyp_margin0}, the empirical distribution $\widehat{P}_N$ is also strictly stochastically transitive with overwhelming probability, and that the expectation of the excess of risk of empirical Kemeny medians decays at an exponential rate, see Proposition 14 therein. In this case, the nearly optimal solution $\sigma^*_{\widehat{P}_N}$ can be made explicit and straightforwardly computed using Eq. \eqref{eq:sol_SST} based on the empirical pairwise probabilities:
$$\widehat{p}_{i,j}=\frac{1}{N}\sum_{s=1}^N\mathbb{I}\{ \Sigma_s(i)<\Sigma_s(j)  \}.$$

 As shall be shown below, the quantity $L_P(\sigma)$ can be seen as a Wasserstein distance between $P$ and the Dirac mass $\delta_{\sigma}$, so that Kemeny consensus ranking can thus be viewed as a radical dimensionality reduction procedure, summarizing $P$ by its closest Dirac measure w.r.t. the distance on the set of probability distributions on $\mathfrak{S}_n$ aforementioned. The general framework for dimensionality reduction developed in the next subsection can be viewed as an extension of consensus ranking.

\subsection{A Mass Transportation Approach to Dimensionality Reduction on $\Sn$}
We now develop a framework, that is shown to extend consensus ranking, for \textit{dimensionality reduction} fully tailored to ranking data exhibiting a specific type of \textit{sparsity}. For this purpose, we consider  the so-termed \textit{mass
transportation} approach to defining metrics on the set of probability distributions on $\mathfrak{S}_n$ as follows, see \cite{Rachev91} (incidentally, this approach is also used in \cite{CJ10} to introduce a specific relaxation of the consensus ranking problem).

\begin{definition}
Let $d:\Sn^2\rightarrow \mathbb{R}_+$ be a metric on $\Sn$ and $q\geq 1$. The $q$-th Wasserstein metric with $d$ as cost function between two probability distributions $P$ and $P'$ on $\mathfrak{S}_n$ is given by:
\begin{equation} \label{eq:metric}
W_{d,q}\left(P,P'  \right)=\inf_{\Sigma\sim P,\; \Sigma' \sim P' }\mathbb{E}\left[ d^q(\Sigma,\Sigma') \right],
\end{equation}
where the infimum is taken over all possible couplings\footnote{Recall that a coupling of two probability distributions $Q$ and $Q'$ is a pair $(U,U')$ of random variables defined on the same probability space such that the marginal distributions of $U$ and $U'$ are $Q$ and $Q'$.} $(\Sigma,\Sigma')$ of $(P,P')$.
\end{definition}
As revealed by the following result, when the cost function $d$ is equal to the Kendall's $\tau$ distance, which case the subsequent analysis focuses on, the Wasserstein metric is bounded by below by the $l_1$ distance between the pairwise probabilities.

\begin{lemma}\label{lem:Kendall} For any probability distributions $P$ and $P'$ on $\mathfrak{S}_n$:
\begin{equation}\label{eq:lower}
W_{d_{\tau},1}\left(P,P'  \right)\geq \sum_{i<j}\vert p_{i,j}-p'_{i,j} \vert.
\end{equation}
The equality holds true when the distribution $P'$ is deterministic (\textit{i.e.} when $\exists \sigma\in \Sn$ s.t. $P'=\delta_{\sigma}$).
\end{lemma}
The proof of Lemma \ref{lem:Kendall} as well as discussions on alternative cost functions (the Spearman $\rho$ distance) are deferred to the Supplementary Material. As shown below, \eqref{eq:lower} is actually an equality for various distributions $P'$ built from $P$ that are of special interest regarding dimensionality reduction.\\


\noindent {\bf Sparsity and Bucket Orders.} Here, we propose a way of describing a distribution $P$ on $\Sn$, originally described by $n!-1$ parameters, by finding a much simpler distribution that approximates $P$ in the sense of the Wasserstein metric introduced above under specific assumptions, extending somehow the consensus ranking concept.
Let $2 \leq K\leq n$ and $\C=(\C_1,\; \ldots,\; \C_K)$ be a \textit{bucket order} of $\n$ with $K$ buckets. In order to gain insight into the rationale behind the approach we promote, observe that a distribution $P'$ can be naturally said to be \textit{sparse} if, for all $1\leq k<l\leq K$ and all $(i,j)\in \C_k\times \C_l$ (i.e. $i \prec_{\C} j $), we have $p'_{j,i}=0$, which means that with probability one $\Sigma'(i)<\Sigma'(j)$, when $\Sigma' \sim P'$. In other words, the relative order of two items belonging to two different buckets is deterministic. Throughout the paper, such a probability distribution is referred to as a \textit{bucket distribution} associated to $\C$. Since the variability of a bucket distribution corresponds to the variability of its marginals within the buckets $\C_k$'s, the set $\P_{\C}$ of all bucket distributions associated to $\C$ is of dimension $d_{\C}=\prod_{k\leq K}\#\C_k!-1\leq n!-1$.
A best summary in $\mathbf{P}_{\C}$ of a distribution $P$ on $\Sn$, in the sense of the Wasserstein metric \eqref{eq:metric}, is then given by any solution $P^*_{\C}$ of the minimization problem
\begin{equation}\label{eq:min_transp}
\min_{P'\in \mathbf{P}_{\C}}W_{d_{\tau},1}(P,P').
\end{equation}
Set $\Lambda_{P}(\C)=\min_{P'\in \mathbf{P}_{\C}}W_{d_{\tau},1}(P,P')$ for any bucket order $\C$.\\

\begin{sloppypar}
\noindent {\bf Dimensionality Reduction.} Let $K\leq n$. We denote by $\mathbf{C}_K$ the set of all bucket orders $\C$ of $\n$ with $K$ buckets. If $P$ can be accurately approximated by a probability distribution associated to a bucket order with $K$ buckets, a natural dimensionality reduction approach consists in finding a solution $\C^{*(K)}$ of
 \begin{equation}\label{eq:best_bucket}
\min_{\C\in \mathbf{C}_K}\Lambda_{P}(\C),
\end{equation}
as well as a solution $P^*_{\C^{*(K)}}$ of
\eqref{eq:min_transp} for $\C=\C^{*(K)}$ and a coupling $(\Sigma,\Sigma_{\C^{*(K)}})$ s.t. $\mathbb{E}[d_{\tau}(\Sigma,\Sigma_{\C^{*(K)}})]=\Lambda_{P}(\C^{*(K)})$.\\
\end{sloppypar}

\noindent{\bf Connection with Consensus Ranking.}
Observe that $\cup_{\C\in \mathbf{C}_n}\mathbf{P}_{\C}$ is the set of all Dirac distributions $\delta_{\sigma}$, $\sigma\in \Sn$. Hence, in the case $K=n$, dimensionality reduction as formulated above boils down to solve Kemeny consensus ranking. Indeed, we have:
$\forall \sigma\in \mathfrak{S}_n$, $W_{d_{\tau},1}\left(P,\delta_{\sigma}  \right)=L_{P}(\sigma)$.
Hence, medians $\sigma^*$ of a probability distribution $P$ (\textit{i.e.} solutions of \eqref{eq:median_pb}) correspond to the Dirac distributions $\delta_{\sigma^*}$ closest to $P$ in the sense of the Wasserstein metric \eqref{eq:metric}: $P^*_{\C^{*(n)}}=\delta_{\sigma^*}$ and $\Sigma_{\C^{*(n)}}=\sigma^*$.
Whereas the space of probability measures on $\Sn$ is of explosive dimension $n!-1$, consensus ranking can be thus somehow viewed as a radical dimension reduction technique, where the original distribution is summarized
by a median permutation $\sigma^*$. In constrast, the other extreme case $K=1$ corresponds to no dimensionality reduction at all, \textit{i.e.} $\Sigma_{\C^{*(1)}}=\Sigma$.

\subsection{Optimal Couplings and Minimal Distortion}
Fix a bucket order $\C=(\C_1,\; \ldots,\; \C_K)$. A simple way of building a distribution in $\mathbf{P}_{\C}$ based on $P$ consists in considering the random ranking $\Sigma_{\C}$ coupled with $\Sigma$, that ranks the elements of any bucket $\C_k$ in the same order as $\Sigma$ and whose distribution $P_{\C}$ belongs to $\mathbf{P}_{\C}$:
\begin{equation}
\forall k\in\{1,\; \ldots,\; K  \},\; \forall i\in\C_k,\;\; \Sigma_{\C}(i)=1+\sum_{l<k}\#\C_l+\sum_{j\in \C_k}\mathbb{I}\{\Sigma(j)<\Sigma(i) \},
\end{equation}
which defines a permutation. Distributions $P$ and $P_{\C}$ share the same marginals within the $\C_k$'s and thus have the same intra-bucket pairwise probabilities $(p_{i,j})_{(i,j)\in \mathcal{C}_k^2}$, for all $k\in\{1,\dots,K\}$. Observe that the expected Kendall's $\tau$ distance between $\Sigma$ and $\Sigma_{\C}$ is given by:
\begin{equation}\label{eq:expect_dist_coupling}
\mathbb{E}\left[d_{\tau}\left(\Sigma,\Sigma_{\C} \right)\right]=\sum_{i\prec_{\C}j}p_{j,i}=\sum_{1\leq k<l\leq K}\sum_{(i,j)\in \C_k\times \C_l}p_{j,i},
\end{equation}
which can be interpreted as the expected number of pairs for which $\Sigma$ violates the (partial) strict order defined by the bucket order $\mathcal{C}$. The result stated below shows that $(\Sigma,\Sigma_{\C})$ is \textit{optimal} among all couplings between $P$ and distributions in $\mathbf{P}_{\C}$ in the sense where \eqref{eq:expect_dist_coupling} is equal to the minimum of~\eqref{eq:min_transp}, namely $\Lambda_{P}(\C)$.
\begin{proposition}\label{prop:kendall_prop}
Let $P$ be any distribution on $\Sn$. For any bucket order $\C=(\C_1,\; \ldots,\; \C_K)$, we have:
\begin{equation}\label{eq:mt_criterion}
\Lambda_{P}(\C)=\sum_{i\prec_{\C}j}p_{j,i}.
\end{equation}
\end{proposition}
The proof, given in the Supplementary Material, reveals that \eqref{eq:lower} in Lemma \ref{lem:Kendall} is actually an equality when $P'=P_{\C}$
and that $\Lambda_{P}(\C)=W_{d_{\tau},1}\left(P,P_{\C}  \right)=\mathbb{E}\left[d_{\tau}\left(\Sigma,\Sigma_{\C} \right)\right]$. Attention must be paid that it is quite remarkable that, when the Kendall's $\tau$ distance is chosen as cost function, the distortion measure introduced admits a simple closed-analytical form, depending on elementary marginals solely, the pairwise probabilities namely. Hence, the distortion of any bucket order can be straightforwardly estimated from independent copies of $\Sigma$, opening up to the design of practical dimensionality reduction techniques based on empirical distortion minimization, as investigated in the next section.
The case where the cost is the Spearman $\rho$ distance is also discussed in the Supplementary Material: it is worth noticing that, in this situation as well,  the distortion can be expressed in a simple manner, as a function of triplet-wise probabilities namely.

\begin{property}\label{prop:comp}
Let $P$ be stochastically transitive. A bucket order $\C=(\C_1,\; \ldots,\; \C_K)$ is said to agree with Kemeny consensus iff we have: $i \prec_{\C} j$ (i.e. $\exists k<l$, $(i,j)\in \C_k\times \C_l$) $\Rightarrow p_{j,i}\leq 1/2$. 
\end{property}
As recalled in the previous subsection, the quantity $L^*_P$ can be viewed as a natural dispersion measure of distribution $P$ and can be expressed as a function of the $p_{i,j}$'s as soon as $P$ is stochastically transitive. The remarkable result stated below shows that, in this case and for any bucket order $\C$ satisfying Property \ref{prop:comp}, $P$'s dispersion can be decomposed as the sum of the (reduced) dispersion of the simplified distribution $P_{\C}$ and the minimum distortion $\Lambda_P(\C)$.

\begin{corollary}
\label{cor:median_lambda}
Suppose that $P$ is stochastically transitive. Then, for any bucket order $\C$ that agrees with Kemeny consensus, we have:
\begin{equation}
L^*_{P}=L^*_{P_{\C}}+\Lambda_P(\C).
\end{equation}
\end{corollary}
In the case where $P$ is strictly stochastically transitive, the Kemeny median $\sigma^*_{P}$ of $P$ is unique (see \cite{CKS17}). If $\C$ fulfills Property \ref{prop:comp}, it is also obviously the Kemeny median of the bucket distribution $P_{\C}$.
As shall be seen in the next section, when $P$ fulfills a strong version of the stochastic transitivity property, optimal bucket orders $\C^{*(K)}$ necessarily agree with the Kemeny consensus, which may greatly facilitates their statistical recovery.

\subsection{Related Work}

The dimensionality reduction approach developed in this paper is connected with
the \textit{optimal bucket order} (OBO) problem considered in the literature, see \textit{e.g.} \cite{aledo2017utopia}, \cite{aledo2018approaching}, \cite{feng2008discovering}, \cite{gionis2006algorithms}, \cite{ukkonen2009randomized}. Given the pairwise probabilities ($p_{i,j})_{1 \le i\ne j\le n}$ of a distribution $P$ over $\Sn$, solving the OBO problem consists in finding a bucket order $\C=(\C_1,\; \ldots,\; \C_K)$ that minimizes the following cost:
\begin{equation}\label{eq:OBO2}
\widetilde{\Lambda}_P(\C) = \sum_{i \ne j} |p_{i,j} - \widetilde{p}_{i,j}|,
\end{equation}
where $\widetilde{p}_{i,j}=1$ if $i \prec_{\C} j$, $\widetilde{p}_{i,j}=0$ if $j \prec_{\C} i$ and $\widetilde{p}_{i,j}=1/2$ if $i \sim_{\C}j$.
In other words, the $\widetilde{p}_{i,j}$'s are the pairwise marginals of the bucket distribution $\widetilde{P}_{\C}$ related to $\C$
with independent and uniformly distributed partial rankings $\Pi_{\C_k}(\widetilde{\Sigma}_{\C})$'s for $\widetilde{\Sigma}_{\C}\sim\widetilde{P}_{\C}$.
Moreover, this cost verifies:
\begin{equation}\label{eq:OBO}
\widetilde{\Lambda}_P(\C) =  2\Lambda_P(\C) + \sum_{k=1}^K\sum_{(i,j)\in \C_k^2}\vert p_{i,j}-1/2\vert.
\end{equation}
Observe that solving the OBO problem is much more restrictive than the framework we developed, insofar as no constraint is set about the intra-bucket marginals of the summary distributions solutions of \eqref{eq:best_bucket}.
Another related work is documented in \cite{shah2016feeling, pananjady2017worst} and develops the concept of \textit{indifference sets}. Formally, a family of pairwise probabilities $(\widetilde{p}_{i,j})$ is said to satisfy the indifference set partition (or bucket order) $\C$ when:
\begin{equation}\label{eq:indifference_set_cdt}
\widetilde{p}_{i,j}=\widetilde{p}_{i',j'} \text{ for all quadruples } (i,j,i',j') \text{ such that } i \sim_{\C} i' \text{ and } j \sim_{\C} j',
\end{equation}
which condition also implies that the intra-bucket marginals are s.t. $\widetilde{p}_{i,j}=1/2$ for $i \sim_{\C} j$ (take $i'=j$ and $j'=i$ in~\eqref{eq:indifference_set_cdt}).
Though related, our approach significantly differs from these works, since it avoids stipulating arbitrary distributional assumptions. For instance, it permits in contrast to test \textit{a posteriori}, once the best bucket order $\C^{*(K)}$ is determined for a fixed $K$, statistical hypotheses such as the independence of the bucket marginal components (\textit{i.e.} $\Pi_{\C_k^{*(K)}}(\Sigma)$'s ) or the uniformity of certain bucket marginal distributions. 
A summary distribution, often very informative and of small dimension both at the same time, is the marginal of the first bucket $\C_1^{*(K)}$ (the top-$m$ rankings where $m=|\C_1^{*(K)}|$).

\section{Empirical Distortion Minimization - Rate Bounds and Model Selection}\label{sec:learning}
In order to recover optimal bucket orders, based on the observation of a training sample $\Sigma_1,\; \ldots,\; \Sigma_N$ of independent copies of $\Sigma$,
Empirical Risk Minimization, the major paradigm of statistical learning, naturally suggests to consider bucket orders $\C=(\C_1,\; \ldots,\; \C_K)$ minimizing the empirical version of the distortion \eqref{eq:mt_criterion}
\begin{equation}\label{eq:emp_distort}
\widehat{\Lambda}_N(\C)=\sum_{i\prec_{\C}j}\widehat{p}_{j,i}=\Lambda_{\widehat{P}_N}(\C),
\end{equation}
where the $\widehat{p}_{i,j}$'s are the pairwise probabilities of the empirical distribution.
For a given shape $\lambda$, we define the Rademacher average
$$
\mathcal{R}_N(\lambda)=\mathbb{E}_{\epsilon_1,\dots,\epsilon_N}\left[\max_{\C\in\mathbf{C}_{K,\lambda}}\frac{1}{N}\left\vert \sum_{s=1}^N\epsilon_s \sum_{i\prec_{\C} j}\mathbb{I}\{\Sigma_s(j)<\Sigma_s(i)\} \right\vert\right],
$$
where $\epsilon_1,\; \ldots,\; \epsilon_N$ are i.i.d. Rademacher r.v.'s (\textit{i.e.} symmetric sign random variables), independent from the $\Sigma_s$'s.
Fix the number of buckets $K\in\{1,\; \ldots,\; n\}$, as well as the bucket order shape $\lambda=(\lambda_1,\; \ldots,\; \lambda_K)\in \mathbb{N}^{*K}$ such that $\sum_{k=1}^K\lambda_k=n$. We recall that $\mathbf{C}_K=\cup_{\lambda'=(\lambda'_1,\dots,\lambda'_K) \in \mathbb{N}^{*K} \text{ s.t. } \sum_{k=1}^K\lambda'_k=n} \mathbf{C}_{K,\lambda'}$. The result stated below describes the generalization capacity of solutions of the minimization problem
\begin{equation}\label{eq:EDM}
\min_{\C\in \mathbf{C}_{K,\lambda}}\widehat{\Lambda}_N(\C),
\end{equation}
over the class $\mathbf{C}_{K,\lambda}$ of bucket orders $\C= (\mathcal{C}_1, \dots, \mathcal{C}_K)$ of shape $\lambda$ (\textit{i.e.} s.t. $\lambda=(\#\mathcal{C}_1, \dots, \#\mathcal{C}_K)$), through a rate bound for their excess of distortion. Its proof is given in the Supplementary material.
\begin{theorem}\label{thm:EDM}
Let $\widehat{C}_{K,\lambda}$ be any empirical distortion minimizer over $\mathbf{C}_{K,\lambda}$, i.e solution of~\eqref{eq:EDM}. Then, for all $\delta\in (0,1)$, we have with probability at least $1-\delta$:
$$
\Lambda_P(\widehat{C}_{K,\lambda})-\inf_{\C\in \mathbf{C}_K}\Lambda_P(\C) \le 4\mathbb{E}\left[\mathcal{R}_N(\lambda)\right] + \kappa(\lambda)\sqrt{\frac{2\log(\frac{1}{\delta})}{N}} +\left\{ \inf_{\C\in \mathbf{C}_{K,\lambda}}\Lambda_P(\C)-\inf_{\C\in \mathbf{C}_{K}}\Lambda_P(\C)\right\},
$$
where $\kappa(\lambda)=\sum_{k=1}^{K-1}\lambda_k\times (n-\lambda_1-\ldots-\lambda_k)$.
\end{theorem}
\begin{sloppypar}
We point out that the Rademacher average is of order $O(1/\sqrt{N})$:
$\mathcal{R}_N(\lambda)\le \kappa(\lambda)\sqrt{2\log\left(\binom{n}{\lambda}\right)/N}$ with $\binom{n}{\lambda}=n!/(\# \C_1 ! \times \cdots \times \# \C_K !) = \# \mathbf{C}_{K,\lambda}$, where $\kappa(\lambda)$ is the number of terms involved in~\eqref{eq:mt_criterion}-\eqref{eq:emp_distort} and $\binom{n}{\lambda}$ is the multinomial coefficient, \textit{i.e.} the number of bucket orders of shape $\lambda$.
Putting aside the approximation error, the rate of decay of the distortion excess is classically of order $O_{\mathbb{P}}(1/\sqrt{N})$.
\end{sloppypar}

\begin{sloppypar}
\begin{remark}{\sc (Empirical Distortion Minimization over $\mathbf{C}_{K}$)} We point out that rate bounds describing the generalization ability of minimizers of \eqref{eq:emp_distort}
over the whole class $\mathbf{C}_{K}$ can be obtained using a similar argument. A slight modification of Theorem \ref{thm:EDM}'s proof shows that, with probability larger than $1-\delta$, their excess of distortion is less than $n^2(K-~1)/K\sqrt{\log(n^2(K-1)\#\mathbf{C}_{K}/(K\delta))/(2N)}$.
Indeed, denoting by $\lambda_{\C}$ the shape of any bucket order $\C$ in $\mathbf{C}_K$, $\max_{\C\in \mathbf{C}_K}\kappa(\lambda_{\C})\le n^2(K-1)/(2K)$, the upper bound being attained when $K$ divides $n$ for $\lambda_1=\dots=\lambda_K=n/K$.
In addition, we have: $\#\mathbf{C}_{K}=\sum_{k=0}^K(-1)^{K-k}\binom{K}{k}k^n$.
\end{remark}
\end{sloppypar}
\begin{sloppypar}
\begin{remark}\label{rk:alt_setup}{\sc (Alternative statistical framework)} Since the distortion \eqref{eq:mt_criterion}  involves pairwise comparisons solely, an empirical version could be computed in a statistical framework stipulating that the observations are of pairwise nature, $(\mathbb{I}\{\Sigma_1(i_1)<\Sigma_1(j_1)\},\; \ldots,\; \mathbb{I}\{\Sigma_N(i_N)<\Sigma_N(j_N)\})$, where $\{(i_s,\; j_s),\;  s=1,\; \ldots,\; N\}$, are i.i.d. pairs, independent from the $\Sigma_s$'s, drawn from an unknown distribution $\nu$ on the set $\{(i,j):\; 1\leq i<j\leq n\}$ such that $\nu(\{(i,j)\})>0$ for all $i<j$. Based on these observations, more easily available in most practical applications (see e.g. \cite{chen2013pairwise}, \cite{park2015preference}), the pairwise probability $p_{i,j}$, $i<j$, can be estimated by:
\begin{equation*}
\frac{1}{N_{i,j}}\sum_{s=1}^N\mathbb{I}\{ (i_s,j_s)=(i,j),\;  \Sigma_s(i_s)<\Sigma_s(j_s) \},
\end{equation*}
with $N_{i,j}=\sum_{s=1}^N\mathbb{I}\{ (i_s,j_s)=(i,j)\}$ and the convention $0/0=0$.
\end{remark}
\end{sloppypar}

\begin{remark}{\sc (Low-dimensional representations)}
For any ranking agent described by its intrinsic preferences $\Sigma\sim P$, the challenge of dimensionality reduction consists in avoiding fully observing $\Sigma$.
Given a solution $\widehat{C}_{K,\lambda}$ of \eqref{eq:EDM}, by only asking to the ranking agent to order items inside each bucket $\widehat{C}_{K,\lambda,k}$ for $k\in\{1,\dots,K\}$,
one can reconstruct the associated optimal ranking $\Sigma_{\widehat{C}_{K,\lambda}}$ coupled with $\Sigma$ and verifying (see Eq. \eqref{eq:expect_dist_coupling}):
$$
\mathbb{E}_{\Sigma\sim P}\left[d_{\tau}\left(\Sigma,\Sigma_{\widehat{C}_{K,\lambda}} \right) \Big| \widehat{C}_{K,\lambda} \right] = \Lambda_P(\widehat{C}_{K,\lambda}).
$$
In other words, the expected approximation error (in terms of Kendall's $\tau$ distance) for observing $\Sigma_{\widehat{C}_{K,\lambda}}$ instead of $\Sigma$ is $\Lambda_P(\widehat{C}_{K,\lambda})$,
which is controlled by the generalization bound given in Theorem \ref{thm:EDM}.
This approach actually corresponds to sampling w.r.t. $P_{\widehat{C}_{K,\lambda}}$ instead of $P$, their Wasserstein distance being $W_{d_{\tau},1}\left(P,P_{\widehat{C}_{K,\lambda}}  \right) = \Lambda_P(\widehat{C}_{K,\lambda})$.
\end{remark}

\noindent {\bf Selecting the shape of the bucket order.} A crucial issue in dimensionality reduction is to determine the dimension of the simpler representation of the distribution of interest.  Here we consider a complexity regularization method to select the bucket order shape $\lambda$ that uses a data-driven penalty based on Rademacher averages. 
Suppose that a sequence $\{(K_m,\lambda_m)\}_{1\le m\le M}$ of bucket order sizes/shapes is given (observe that $M\le \sum_{K=1}^n \binom{n-1}{K-1}=2^{n-1}$). In order to avoid overfitting, consider the complexity penalty given by
\begin{equation}\label{eq:penalty}
{\sc pen}(\lambda_m,N)=2\mathcal{R}_N(\lambda_m)
\end{equation}
and the minimizer $\widehat{\C}_{K_{\widehat{m}}, \lambda_{\widehat{m}}}$ of the penalized empirical distortion, with
\begin{equation}\label{eq:EDM_penalty}
\widehat{m}=\argmin_{1\le m\le M}\left\{\widehat{\Lambda}_N(\widehat{\C}_{K_m,\lambda_m})+ {\sc pen}(\lambda_m,N)\right\} \text{ and } \widehat{\Lambda}_N(\widehat{\C}_{K,\lambda})=\min_{\C\in \mathbf{C}_{K,\lambda}} \widehat{\Lambda}_N(\C).
\end{equation}
The next result shows that the bucket order thus selected nearly achieves the performance that would be obtained with the help of an oracle, revealing the value of the index $m$ ruling the bucket order size/shape that minimizes $\mathbb{E}[\Lambda_P(\widehat{\C}_{K_m,\lambda_m})]$.

\begin{theorem} \label{thm:select}{\sc (An oracle inequality)} Let $\widehat{\C}_{K_{\widehat{m}}, \lambda_{\widehat{m}}}$ be any penalized empirical distortion minimizer over $\mathbf{C}_{K_{\widehat{m}},\lambda_{\widehat{m}}}$, i.e solution of~\eqref{eq:EDM_penalty}. Then we have:
\begin{equation*}
\mathbb{E}\left[\Lambda_P(\widehat{\C}_{K_{\widehat{m}},\lambda_{\widehat{m}}})\right]\leq \min_{1\le m\le M}\left\{\mathbb{E}\left[\Lambda_P(\widehat{\C}_{K_m,\lambda_m})\right]+ 2\mathbb{E}\left[\mathcal{R}_N(\lambda_m)\right] \right\}+ 5M \binom{n}{2} \sqrt{\frac{\pi}{2N}}.
\end{equation*}
\end{theorem}

\noindent {\bf The Strong Stochastic Transitive Case.}
The theorem below shows that, when strong/strict stochastic transitivity properties hold for the considered distribution $P$,  optimal buckets are those which agree with the Kemeny median.
\begin{theorem}
\label{thm:bucket_median}
Suppose that $P$ is strongly/strictly stochastically transitive. Let $K\in\{1,\; \ldots,\; n  \}$ and $\lambda=(\lambda_1,\; \ldots,\; \lambda_K)$ be a given bucket size and shape. Then, the minimizer of the distortion $\Lambda_{P}(\C)$ over $\mathbf{C}_{K,\lambda}$ is unique and given by  $\C^{*(K,\lambda)}=(\C^{*(K,\lambda)}_1,\; \ldots,\; \C^{*(K,\lambda)}_K)$, where
\begin{equation}\label{eq:opt_bucket}
\C^{*(K,\lambda)}_{k}=\left\{i\in \n: \;\; \sum_{l<k} \lambda_l< \sigma^*_P(i)\leq  \sum_{l\leq k} \lambda_l \right\} \text{ for } k\in\{1,\; \ldots,\; K\}.
\end{equation}
In addition, for any $\C\in \mathbf{C}_{K,\lambda}$, we have:
\begin{equation}\label{eq:excess_distort}
\Lambda_{P}(\C)-\Lambda_P(\C^{*(K,\lambda)})\geq 2\sum_{j\prec_{\C}i}( 1/2-p_{i,j}) \cdot \mathbb{I}\{ p_{i,j}<1/2 \}.
\end{equation}
\end{theorem}
In other words, $\C^{*(K,\lambda)}$ is the unique bucket in $\mathbf{C}_{K,\lambda}$ that agrees with $\sigma^*_P$ (\textit{cf} Property \ref{prop:comp}).
Hence, still under the hypotheses of Theorem \ref{thm:bucket_median}, the minimizer $\C^{*(K)}$ of \eqref{eq:best_bucket} also agrees with $\sigma^*_P$
and corresponds to one of the $\binom{n-1}{K-1}$ possible segmentations of the ordered list $(\sigma_P^{*-1}(1),\dots,\sigma_P^{*-1}(n))$ into $K$ segments. This property paves the way
to design efficient procedures, such as the \textsc{BuMeRank} algorithm described in Appendix B, for recovering bucket order representations with a fixed distortion rate of minimal dimension, avoiding to specify the size/shape in advance. If, in addition, condition \eqref{eq:hyp_margin0} is fulfilled, when $\widehat{P}_N$ is strictly stochastically transitive (which then happens with overwhelming probability, see Proposition 14 in \cite{CKS17}), the computation of the empirical Kemeny median $\sigma^*_{\widehat{P}_N}$ is immediate from formula \eqref{eq:sol_SST} (replacing $P$ by $\widehat{P}_N$), as well as an estimate of $\C^{*(K,\lambda)}$, plugging $\sigma^*_{\widehat{P}_N}$ into \eqref{eq:opt_bucket} as implemented in the experiments below. When the empirical distribution $\widehat{P}_N$ is not stochastically transitive, which happens with negligible probability, the empirical median can be classically replaced by any permutation obtained from the Copeland score by breaking ties at random.
The following result shows that, in the strict/strong stochastic transitive case, when the low-noise condition $\textbf{NA}(h)$ is fulfilled, the excess of distortion of the empirical minimizers is actually of order $O_{\mathbb{P}}(1/N)$.

\begin{theorem}\label{thm:fast}{(\sc Fast rates)}
Let $\lambda$ be a given bucket order shape and $\widehat{C}_{K,\lambda}$ any empirical distortion minimizer over $\mathbf{C}_{K,\lambda}$. Suppose that $P$ is strictly/strongly stochastically transitive and fulfills condition \eqref{eq:hyp_margin0}. Then, for any $\delta>0$, we have with probability $1-\delta$:
$$
\Lambda_P(\widehat{\C}_{K,\lambda})-\Lambda_P(\C^{*(K,\lambda)})\leq \left(\frac{2^{\binom{n}{2}+1}n^2}{h}\right) \times \frac{\log\left( \binom{n}{\lambda}/\delta \right)}{N}.
$$
\end{theorem}
The proof is given in the Appendix section.

\section{Numerical Experiments on Real-world Datasets}\label{sec:exp}


\begin{figure}[htp]

\centering
\includegraphics[width=.325\textwidth]{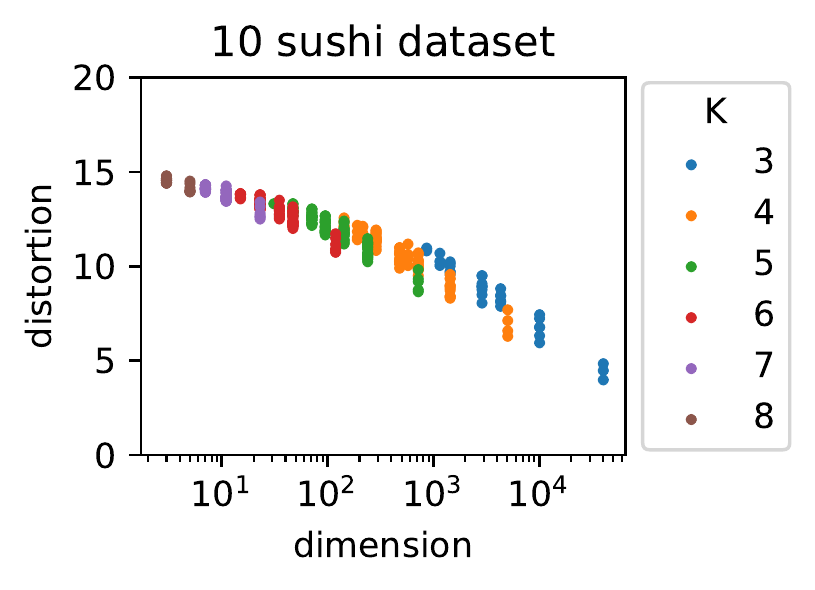}
\includegraphics[width=.325\textwidth]{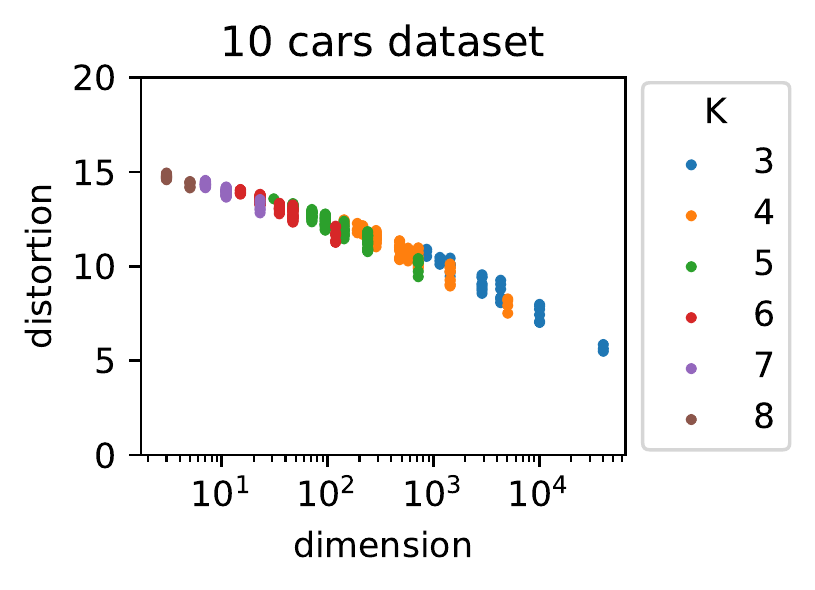}\hfill
\includegraphics[width=.335\textwidth]{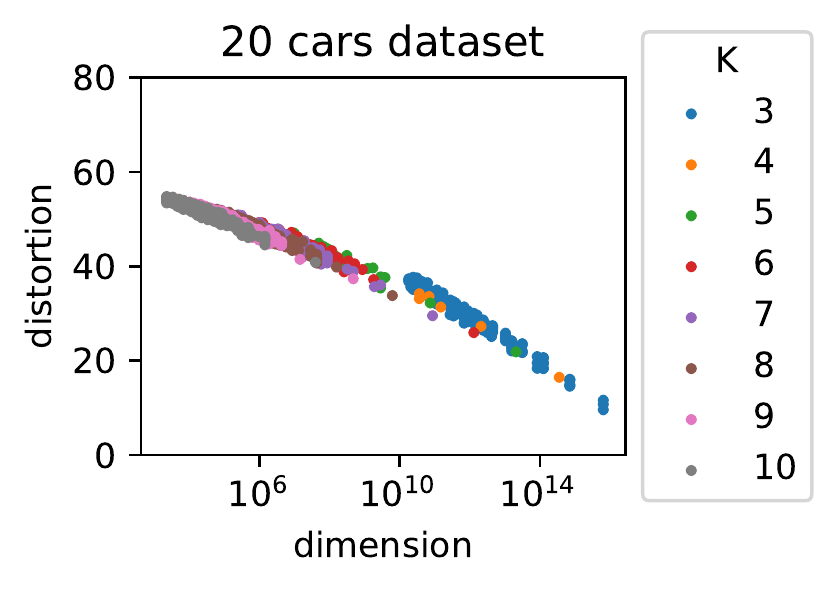}\hfill
\caption{Dimension-Distortion plot for different bucket sizes on real-world preference datasets.}
\label{fig:real_data}
\end{figure}
\noindent In this section we illustrate the relevance of our approach through real-world ranking datasets, which exhibit the type of sparsity considered in the present article.
The first one is the well-known Sushi dataset (see~\cite{kamishima2003nantonac}), which consists of full rankings describing the preferences of $N = 5000$ individuals over $n=10$ sushi dishes.
We also considered the two \textit{Cars preference datasets}\footnote{\url{http://users.cecs.anu.edu.au/~u4940058/CarPreferences.html}, First experiment.} (see \cite{Abbasnejad2013}).
It consists of pairwise comparisons of users between $n$ different cars. In the first dataset, $60$ users are asked to make all the possible $45$ pairwise comparisons between $10$ cars (around $3000$ samples).
In the second one, $60$ users are asked to make (randomly selected) $38$ comparisons between $20$ cars (around $2500$ samples).
For each dataset, the empirical ranking $\sigma^*_{\widehat{P}_N}$ is computed based on the empirical pairwise probabilities. 
In Figure \ref{fig:real_data}, the dimension $d_{\C}$ (in logarithmic scale) \textit{vs} distortion $\widehat{\Lambda}_N(\C)$ diagram is plotted for each dataset, for several bucket sizes ($K$) and shapes ($\lambda$). These buckets are obtained by segmenting $\sigma^*_{\widehat{P}_N}$ with respect to $\lambda$ as explained at the end of the previous section.
Each color on a plot corresponds to a specific size $K$, and each point in a given color thus represents a bucket order of size $K$. As expected, on each plot the lowest distortion is attained for high-dimensional buckets (i.e., of smaller size $K$).
These numerical results shed light on the sparse character of these empirical ranking distributions. Indeed, the dimension $d_{\C}$ can be drastically reduced, by choosing the size $K$ and shape $\lambda$ in an appropriate manner, while keeping a low distortion for the representation.
The reader may refer to the Appendix for additional dimension/distortion plots for different distributions which underline the sparsity observed here: specifically, these empirical distributions show intermediate behaviors between a true bucket distribution and a uniform distribution (i.e., without exhibiting bucket sparsity). The code to reproduce our results is available: \url{https://github.com/akorba/Dimensionality_Reduction_Ranking/}.

\section{Conclusion}\label{sec:concl}

In this paper, we have developed novel theoretical concepts to represent efficiently \textit{sparse} ranking data distributions. We have introduced a distortion measure, based on a mass transportation metric on the set of probability distributions on the set of rankings (with Kendall's $\tau$ as transportation cost) in order to evaluate the accuracy of (bucket) distribution representations and investigated the performance of empirical distortion minimizers. We have also provided empirical evidence that the notion of sparsity, which the dimensionality reduction method proposed relies on, is encountered in various real-world situations. In a future work, we intend to investigate at length how to exploit such sparse representations for improving the completion of certain statistical learning tasks based on ranking data (e.g. clustering, ranking prediction), by circumventing this way the curse of dimensionality.

\acks{This work was supported by the industrial chair \textit{Machine Learning for Big Data} from T\'el\'ecom ParisTech
and by a public grant (\textit{Investissement d'avenir} project, reference ANR-11-LABX-0056-LMH, LabEx LMH).}

\bibliography{biblio}

\newpage

\appendix

\section*{Appendix A - Technical Proofs}

\subsection*{Proof of Lemma \ref{lem:Kendall}}

Consider two probability distributions $P$ and $P'$ on $\mathfrak{S}_n$. Fix $i\neq j$ and let $(\Sigma,\Sigma')$ be a pair of random variables defined on a same probability space, valued in $\mathfrak{S}_n$ and such that $p_{i,j}=~\mathbb{P}_{\Sigma\sim P}\{\Sigma(i)<~\Sigma(j)\}$ and $p'_{i,j}=\mathbb{P}_{\Sigma'\sim P'}\{\Sigma'(i)<\Sigma'(j)\}$.
Set
$$
\pi_{i,j}=\mathbb{P}\left\{ \Sigma'(i)<\Sigma'(j)\mid  \Sigma(i)<\Sigma(j) \right\}.
$$
Equipped with this notation, by the law of total probability, we have:
\begin{equation}\label{eq:margin}
p'_{i,j}=p_{i,j}\pi_{i,j}+(1-p_{i,j})(1-\pi_{j,i}).
\end{equation}
In addition, we may write
\begin{multline*}
\mathbb{E}\left[ d_{\tau}(\Sigma,\Sigma')\right]=\sum_{i<j}\mathbb{E}\left[\mathbb{I}\{ (\Sigma(i)-\Sigma(j))(\Sigma'(i)-\Sigma'(j))<0 \}\right]\\
=\sum_{i<j}\mathbb{E}\left[\mathbb{I}\{\Sigma(i)<\Sigma(j) \}\mathbb{I}\{\Sigma'(i)>\Sigma'(j) \} + \mathbb{I}\{\Sigma(i)>\Sigma(j) \}\mathbb{I}\{\Sigma'(i)<\Sigma'(j) \}\right]\\
=\sum_{i<j} p_{i,j}(1-\pi_{i,j})+(1-p_{i,j})(1-\pi_{j,i}) .
\end{multline*}
Suppose that $p_{i,j}<p'_{i,j}$. Using \eqref{eq:margin}, we have $p_{i,j}(1-\pi_{i,j})+(1-p_{i,j})(1-\pi_{j,i})=p'_{i,j}+(1-2\pi_{i,j})p_{i,j}$, which quantity is minimum when $\pi_{i,j}=1$ (and in this case $\pi_{j,i}=(1-p'_{i,j})/(1-p_{i,j})$), and then equal to $\vert p_{i,j}-p'_{i,j}\vert$. We recall that we can only set $\pi_{i,j}=1$ if the initial assumption $p_{i,j}<p'_{i,j}$ holds. In a similar fashion, if $p_{i,j}>p'_{i,j}$, we have $p_{i,j}(1-\pi_{i,j})+(1-p_{i,j})(1-\pi_{j,i})=2(1-p_{i,j})(1-\pi_{j,i})+p_{i,j}-p'_{i,j}$, which is minimum for $\pi_{j,i}=1$ (we have incidentally $\pi_{i,j}=p'_{i,j}/p_{i,j}$ in this case) and then equal to $\vert p_{i,j}-p'_{i,j}\vert$. Since we clearly have
\begin{multline*}
W_{d_{\tau},1}\left(P,P'  \right) \ge  \\\sum_{i<j} \inf_{(\Sigma,\Sigma') \text { s.t. } \mathbb{P}\{\Sigma(i)<\Sigma(j)\}=p_{i,j} \text{ and } \mathbb{P}\{\Sigma'(i)<\Sigma'(j)\}=p'_{i,j} }  \mathbb{P}\left[ (\Sigma(i)-\Sigma(j))(\Sigma'(i)-\Sigma'(j))<0 \right] ,
\end{multline*}
this proves that
$$
W_{d_{\tau},1}\left(P,P'  \right) \ge \sum_{i<j}\vert p'_{i,j}-p_{i,j} \vert.
$$
As a remark, given a distribution $P$ on $\mathfrak{S}_n$, when $P'=P_{\C}$ with $\C$ a bucket order of $\n$ with $K$ buckets, the optimality conditions on the $\pi_{i,j}$'s are fulfilled by the coupling $(\Sigma,\Sigma_{\C})$,
which implies that:
\begin{equation}\label{eq:min_dist}
W_{d_{\tau},1}\left(P,P_{\C} \right) = \sum_{i<j}\vert p'_{i,j}-p_{i,j} \vert = \sum_{1\leq k<l\leq K}\sum_{(i,j)\in \C_k\times \C_l}p_{j,i},
\end{equation}
where $p'_{i,j} = \mathbb{P}_{\Sigma_{\C}\sim P_{\C}}\left[  \Sigma_{\C}(i)<\Sigma_{\C}(j)\right] = p_{i,j}\mathbb{I}\{k=l\} + \mathbb{I}\{k<l\}$,
with $(k,l)\in\{1,\dots,K\}^2$ such that $(i,j)\in\C_k\times\C_l$.

\subsection*{Proof of Proposition \ref{prop:kendall_prop}}
Let $\C$ be a bucket order of $\n$ with $K$ buckets. Then, for $P'\in \mathbf{P}_{\C}$, Lemma \ref{lem:Kendall} implies that:
\begin{equation*}
    W_{d_{\tau},1}\left(P,P'  \right) \ge \sum_{i<j}\vert p'_{i,j}-p_{i,j} \vert = \sum_{k=1}^K \sum_{i<j, (i,j)\in \C_k^2} \vert p'_{i,j}-p_{i,j} \vert + \sum_{1\leq k<l\leq K}\sum_{(i,j)\in \C_k\times \C_l}p_{j,i},
\end{equation*}
where the last equality results from the fact that $p'_{i,j}=1$ when $(i,j)\in\C_k\times \C_l$ with $k<l$.
When $P'=P_{\C}$,  the intra-bucket terms are all equal to zero. Hence, it results from \eqref{eq:min_dist} that :
$$
W_{d_{\tau},1}\left(P,P_{\C}  \right) = \sum_{1\leq k<l\leq K}\sum_{(i,j)\in \C_k\times \C_l}p_{j,i}=\Lambda_P(\C).
$$

\subsection*{Proof of Theorem \ref{thm:EDM}}


Observe first that the excess of distortion can be bounded as follows:
$$
\Lambda_P(\widehat{C}_{K,\lambda})-\inf_{\C\in \mathbf{C}_K}\Lambda_P(\C)\leq 2\max_{\C\in \mathbf{C}_{K,\lambda}}\left|\widehat{\Lambda}_N(\C)-\Lambda_P(\C)\right| + \left\{ \inf_{\C\in \mathbf{C}_{K,\lambda}}\Lambda_P(\C)-\inf_{\C\in \mathbf{C}_{K}}\Lambda_P(\C)\right\}.
$$
By a classical symmetrization device (see e.g. \cite{vanweak}), we have:
\begin{equation}\label{eq:symmetrization}
\mathbb{E}\left[\max_{\C\in \mathbf{C}_{K,\lambda}} \left| \widehat{\Lambda}_N(\C) - \Lambda_P(\C) \right| \right] \le 2\mathbb{E}\left[\mathcal{R}_N(\lambda)\right].
\end{equation}
Hence, using McDiarmid's inequality, for all $\delta\in (0,1)$ it holds with probability at least $1-\delta$:
\begin{equation*}
\max_{\C\in \mathbf{C}_{K,\lambda}}\left|\widehat{\Lambda}_N(\C)-\Lambda_P(\C)\right| \le 2\mathbb{E}\left[\mathcal{R}_N(\lambda)\right] + \kappa(\lambda)\sqrt{\frac{\log(\frac{1}{\delta})}{2N}}.
\end{equation*}

\subsection*{Proof of Theorem \ref{thm:select}}

Following the proof of Theorem 8.1 in \cite{boucheron2005theory}, we have for all $m\in\{1,\dots,M\}$,
\begin{equation*}
  \begin{split}
    \mathbb{E}\left[\Lambda_P(\widehat{\C}_{K_{\widehat{m}},\lambda_{\widehat{m}}})\right]
    &\le \mathbb{E}\left[\Lambda_P(\widehat{\C}_{K_m,\lambda_m})\right]+\mathbb{E}\left[{\sc pen}(\lambda_m,N)\right]\\
    &\qquad + \sum_{m'=1}^M \mathbb{E}\left[\left( \max_{\C\in \mathbf{C}_{K_{m'},\lambda_{m'}}} \Lambda_P(\C)-\widehat{\Lambda}_N(\C)-{\sc pen}(\lambda_{m'},N) \right)_+\right],
  \end{split}
\end{equation*}
where $x_+=\max(x,0)$ denotes the positive part of $x$.
In addition, for any $\delta > 0$, we have:
\begin{multline*}
    \Prob{\max_{\C\in \mathbf{C}_{K_m,\lambda_m}} \Lambda_P(\C)-\widehat{\Lambda}_N(\C) \ge {\sc pen}(\lambda_m,N) + \delta}\\
    \le \Prob{\max_{\C\in \mathbf{C}_{K_m,\lambda_m}} \Lambda_P(\C)-\widehat{\Lambda}_N(\C) \ge \mathbb{E}\left[\max_{\C\in \mathbf{C}_{K_m,\lambda_m}} \Lambda_P(\C)-\widehat{\Lambda}_N(\C)\right] + \frac{\delta}{5}}\\
    + \Prob{\mathcal{R}_N(\lambda_m) \le \mathbb{E}\left[\mathcal{R}_N(\lambda_m)\right]-\frac{2}{5}\delta}
    \le 2\exp\left(-\frac{2N\delta^2}{25\kappa(\lambda_m)^2}\right),
\end{multline*}
using~\eqref{eq:symmetrization} for the first inequality,
and both McDiarmid's inequality and Lemma 8.2 in \cite{boucheron2005theory} for the second inequality.
Observing that $\kappa(\lambda)\le \binom{n}{2}$ and integrating by parts conclude the proof.

\subsection*{Proof of Theorem \ref{thm:bucket_median}}
Consider a bucket order $\C=(\C_1,\; \ldots,\; \C_K)$ of shape $\lambda$, different from \eqref{eq:opt_bucket}. Hence, there exists at least a pair $\{i,j\}$ such that $j\prec_{\C}i$ and $\sigma_P^*(j)<\sigma_P^*(i)$ (or equivalently $p_{i,j}<1/2$). Consider such a pair $\{i,j\}$. Hence, there exist $1\leq k<l\leq K$ s.t. $(i,j)\in\C_k\times C_l$.
Define the bucket order $\C'$ which is the same as $\C$ except that the buckets of $i$ and $j$ are
swapped: $\C'_k=\{j\}\cup\C_k\setminus\{i\}$, $\C'_l=\{i\}\cup\C_l\setminus\{j\}$ and $\C'_m=\C_m$ if $m\in\{1,\dots,K\}\setminus\{k,l\}$. Observe that
\begin{multline*}
\Lambda_{P}(\C')-\Lambda_{P}(\C)=
p_{i,j}-p_{j,i} + \sum_{a\in\C_k\setminus\{i\}} p_{i,a}-p_{j,a} + \sum_{a\in\C_l\setminus\{j\}} p_{a,j}-p_{a,i}\\
+ \sum_{m=k+1}^{l-1} \sum_{a\in\C_m} p_{a,j}-p_{a,i} + p_{i,a}-p_{j,a} \leq 2( p_{i,j}-1/2)< 0.
\end{multline*}
Considering now all the pairs $\{i,j\}$ such that $j\prec_{\C}i$ and $p_{i,j}<1/2$, it follows by induction that
\begin{equation}\label{eq:lowerb}
\Lambda_{P}(\C)-\Lambda_P(\C^{*(K,\lambda)})\geq 2\sum_{j\prec_{\C}i}( 1/2-p_{i,j}) \cdot \mathbb{I}\{ p_{i,j}<1/2 \}.
\end{equation}


\subsection*{Proof of Theorem \ref{thm:fast}}

The fast rate analysis essentially relies on the following lemma providing a control of the variance of the empirical excess of distortion
$$
\widehat{\Lambda}_N(\C)-\widehat{\Lambda}_N(\C^{*(K,\lambda)})=\frac{1}{N}\sum_{s=1}^N\sum_{i\neq j}\mathbb{I}\{\Sigma_s(j)<\Sigma_s(i)\}\cdot \left(  \mathbb{I}\left\{ i\prec_{\C}j\}-  \mathbb{I}\{i<_{\C^{*(K,\lambda)}}j \right\} \right).
$$
Set $D(\C)=\sum_{i\neq j}\mathbb{I}\{\Sigma(j)<\Sigma(i)\}\cdot \left(  \mathbb{I}\left\{ i\prec_{\C}j\}-  \mathbb{I}\{i<_{\C^{*(K,\lambda)}} j\right\} \right)$. Observe that $\mathbb{E}[D(\C)]=\Lambda_P(\C)-\Lambda_P(\C^{*(K,\lambda)})$.
\begin{lemma}\label{lem:var_control} Let $\lambda$ be a given bucket order shape. We have:
$$
var\left(D(\C)\right)\leq 2^{\binom{n}{2}}(n^2/h)\cdot \mathbb{E}[D(\C)].
$$
\end{lemma}
\begin{proof}
As in the proof of Theorem \ref{thm:bucket_median}, consider a bucket order $\C=(\C_1,\; \ldots,\; \C_K)$ of shape $\lambda$, different from \eqref{eq:opt_bucket}, a pair $\{i,j\}$ such that there exist $1\leq k<l\leq K$ s.t. $(i,j)\in\C_k\times C_l$ and $\sigma_P^*(j)<\sigma_P^*(i)$ and the bucket order $\C'$ which is the same as $\C$ except that the buckets of $i$ and $j$ are
swapped. We have:
\begin{multline*}
D(\C')-D(\C)=\mathbb{I}\{\Sigma(i)<\Sigma(j)  \} - \mathbb{I}\{\Sigma(j)<\Sigma(i)  \} + \sum_{a\in\C_k\setminus\{i\}} \mathbb{I}\{\Sigma(i)<\Sigma(a)  \} -\mathbb{I}\{\Sigma(j)<\Sigma(a)  \}\\  + \sum_{a\in\C_l\setminus\{j\}} \mathbb{I}\{\Sigma(a)<\Sigma(j)  \} -\mathbb{I}\{\Sigma(a)<\Sigma(i)  \} \\
+ \sum_{m=k+1}^{l-1} \sum_{a\in\C_m} \mathbb{I}\{\Sigma(a)<\Sigma(j)  \} -\mathbb{I}\{\Sigma(a)<\Sigma(i)  \}  + \mathbb{I}\{\Sigma(i)<\Sigma(a)  \} -\mathbb{I}\{\Sigma(j)<\Sigma(a)  \}.
\end{multline*}

Hence, we have: $var(D(\C')-D(\C))\leq 4n^2$. By induction, we then obtain that:
\begin{multline*}
var\left(D(\C)\right)\leq 2^{\binom{n}{2}-1}(4n^2)\#\left\{(i,j):\; i\prec_{\C}j \text{ and } p_{j,i}>1/2  \right\}\\
\leq 2^{\binom{n}{2}-1}(4n^2/h)\sum_{j\prec_{\C}i}( 1/2-p_{i,j}) \cdot \mathbb{I}\{ p_{i,j}<1/2 \}\leq 2^{\binom{n}{2}}(n^2/h)\mathbb{E}[D(\C)],
\end{multline*}
by combining \eqref{eq:excess_distort} with condition \eqref{eq:hyp_margin0}.
\end{proof}

Applying Bernstein's inequality to the i.i.d. average $(1/N)\sum_{s=1}^ND_s(\C)$, where
\begin{equation*}
D_s(\C)=\sum_{i\neq j}\mathbb{I}\{\Sigma_s(j)<\Sigma_s(i)\}\cdot \left(  \mathbb{I}\left\{ i\prec_{\C}j\}-  \mathbb{I}\{i<_{\C^{*(K,\lambda)}} j\right\} \right),
\end{equation*}
for $1\leq s\leq N$ and the union bound over the bucket orders $\C$ in $\mathbf{C}_{K,\lambda}$ (recall that $\#\mathbf{C}_{K,\lambda}= \binom{n}{\lambda}$), we obtain that, for all $\delta\in (0,1)$, we have with probability larger than $1-\delta$: $\forall \C\in \mathbf{C}_{K,\lambda}$,
\begin{multline*}
\mathbb{E}[D(\C)]= \Lambda_P(\C)-\Lambda_P(\C^{*(K,\lambda)}) \leq \widehat{\Lambda}_N(\C)-\widehat{\Lambda}_N(\C^{*(K,\lambda)})+\sqrt{\frac{2 var(D(\C))\log\left(\binom{n}{\lambda}/\delta\right)}{N}} \\+\frac{4\kappa(\lambda)\log(\binom{n}{\lambda}/\delta)}{3N}.
\end{multline*}
Since $\widehat{\Lambda}_N(\widehat{C}_{K,\lambda})-\widehat{\Lambda}_N(\C^{*(K,\lambda)}) \leq 0$ by assumption and using the variance control provided by Lemma \ref{lem:var_control} above, we obtain that, with probability at least $1-\delta$, we have:
\begin{multline*}
 \Lambda_P(\widehat{C}_{K,\lambda})-\Lambda_P(\C^{*(K,\lambda)})\leq\sqrt{\frac{2^{\binom{n}{2}+1}n^2\left( \Lambda_P(\widehat{C}_{K,\lambda})-\Lambda_P(\C^{*(K,\lambda)})\right)/h \times \log(\binom{n}{\lambda}/\delta)}{N}}\\+\frac{4\kappa(\lambda)\log(\binom{n}{\lambda}/\delta)}{3N}.
\end{multline*}
Finally, solving this inequality in $ \Lambda_P(\widehat{C}_{K,\lambda})-\Lambda_P(\C^{*(K,\lambda)})$ yields the desired result.

\section*{Appendix B - Hierarchical Recovery of a Bucket Distribution}

Motivated by Theorem \ref{thm:bucket_median}, we propose a hierarchical 'bottom-up' procedure to recover, from ranking data, a bucket order representation (agreeing with Kemeny consensus) of smallest dimension for a fixed level of distortion, that does not requires to specify in advance the bucket size $K$ and thus avoids computing the optimum $\eqref{eq:opt_bucket}$ for all possible shape/size.

Suppose for simplicity that $P$ is strictly/strongly stochastically transitive. One starts with the bucket order of size $n$ defined by its Kemeny median $\sigma^*_P$:
$$
\C(0)=(\{\sigma_P^{*-1}(1)\},\; \dots,\; \{\sigma_P^{*-1}(n)\}).
$$
The initial representation has minimum dimension, \textit{i.e.} $d_{\C(0)}=0$, and maximal distortion among all bucket order representations agreeing with $\sigma^*_P$, \textit{i.e.} $\Lambda_P(\C(0))=L^*_P$, see Corollary \ref{cor:median_lambda}.
The binary agglomeration strategy we propose, namely the {\sc BuMeRank} (for '\textbf{Bu}cket \textbf{Me}rge') algorithm, consists in recursively merging two adjacent buckets $\C_k(j)$ and $\C_{k+1}(j)$ of the current bucket order $\C(j)=(\C_1(j),\; \ldots,\; C_K(j))$ into a single bucket, yielding the 'coarser' bucket order
\begin{equation}
	\label{eq:merged_bucket}
	\C(j+1)=(\C_1(j), \dots, \C_{k-1}(j), \C_k(j)\cup\C_{k+1}(j), \C_{k+2}(j), \dots, \C_{K}(j)).
\end{equation}
The pair $(\C_k(j), \C_{k+1}(j))$ chosen corresponds to that maximizing the quantity
\begin{equation}
	\label{eq:delta}
	\Delta_P^{(k)}(\C(j)) = \sum_{i\in\C_k(j), j\in\C_{k+1}(j)} p_{j, i}.
\end{equation}
The agglomerative stage $\C(j) \to \C(j+1)$ increases the dimension of the representation,
\begin{equation}
	\label{eq:merged_dim}
	d_{\C(j+1)} = (d_{\C(j)}+1)\times\binom{\#\C_k(j)+\#\C_{k+1}(j)}{\#\C_k(j)}-1,
\end{equation}
while reducing the distortion by $\Lambda_P(\C(j))-\Lambda_P(\C(j+1))=\Delta_P^{(k)}(\C(j))$.

\medskip

\fbox{
	\begin{minipage}[t]{13.5cm}
		\medskip

		{\small
			\begin{center}
				{\sc BuMeRank Algorithm}
			\end{center}

			\begin{enumerate}
				\item {\bf Input.} Training data $\{\Sigma_i\}_{i=1}^N$, maximum dimension $d_\text{max}\ge 0$, distortion tolerance $\epsilon\ge 0$.
				\item {\bf Initialization.} Compute empirical Kemeny median $\sigma^*_{\widehat{P}_N}$ and $\C(0) = (\{\sigma_{\widehat{P}_N}^{*-1}(1)\}, \dots, \{\sigma_{\widehat{P}_N}^{*-1}(n)\})$. Set $K\leftarrow n$.
				\item {\bf Iterations.} While $K\ge 3$ and $\widehat{\Lambda}_N(\C(n-K))>\epsilon$,
				\begin{enumerate}
					\item Compute $k\in\argmax_{1\le l\le K-1}\Delta_{\widehat{P}_N}^{(l)}(\C(n-K))$ and $\C(n-K+1)$.
					\item If $d_{\C(n-K+1)}>d_\text{max}$: go to 4. Else: set $K\leftarrow K-1$.
				\end{enumerate}
				\item {\bf Output.} Bucket order $\C(n-K)$.
			\end{enumerate}
		}
		\medskip

	\end{minipage}
}
\medskip

For notational convenience, the \textsc{BuMeRank} algorithm is defined taking full rankings $\Sigma_i$'s as input,
but it remains valid in the pairwise comparisons framework (see Remark \ref{rk:alt_setup}).
This algorithm is specifically designed for finding the bucket order $\C$ of minimal dimension $d_\C$ (i.e. of maximal size $K$)
such that a bucket distribution in $\P_{\C}$ approximates well the original distribution $P$ (i.e. with small distortion $\Lambda_P(\C)$).
The next result formally supports this idea in the limit case of $P$ being a bucket distribution.

\begin{theorem}
	\label{thm:algo}
	Let $P$ be a strongly/strictly stochastically transitive bucket distribution and denote $K^*=\max\{K\in\{2,\dots,n\}, \exists \text{ bucket order } \C \text{ of size } K \text{ s.t. } P\in\P_{\C}\}$.\\
	\noindent $(i)$ There exists a unique $K^*$-shape $\lambda^*$ such that $\Lambda_P(\C^{*(K^*,\lambda^*)})=0$.\\
	\noindent $(ii)$ For any bucket order $\C$ such that $P\in\P_\C$: $\C\neq\C^{*(K^*,\lambda^*)} \Rightarrow d_\C>d_{\C^{*(K^*,\lambda^*)}}$.\\
	\noindent $(iii)$ The {\sc BuMeRank} algorithm, runned with $d_{\text{max}}=n!-1$, $\epsilon=0$ and theoretical quantities ($\sigma^*_P$, $\Delta_P^{(k)}$'s and $\Lambda_P$) instead of estimates,
	outputs $\C^{*(K^*,\lambda^*)}$.
\end{theorem}

\begin{proof}
  Straightforward if $K^*=n$: assume $K^*<n$ in the following.\\
	(i). Existence is ensured by definition of $K^*$ combined with Theorem \ref{thm:bucket_median}.
  Assume there exist two distinct $K^*$-shapes $\lambda$ and $\lambda'$ such that $\Lambda_P(\C^{*(K^*,\lambda)})=\Lambda_P(\C^{*(K^*,\lambda')})=0$.
  Necessarily, there exists $k\in\{1,\dots,K-1\}$ such that, for example, $\C^{*(K^*,\lambda)}_k \cap \C^{*(K^*,\lambda')}_{k+1} \neq \emptyset$ and $\C^{*(K^*,\lambda')}_{k+1} \nsubseteq \C^{*(K^*,\lambda)}_k$.
  Then, define a new bucket order $\widetilde{\C}$ of size $K^*+1$ as follows:
  \begin{multline*}
    \widetilde{\C}=\Biggl(\C^{*(K^*,\lambda')}_1,\dots,\C^{*(K^*,\lambda')}_k,\C^{*(K^*,\lambda)}_k\cap\C^{*(K^*,\lambda')}_{k+1},\\
    \C^{*(K^*,\lambda')}_{k+1}\setminus\left(\C^{*(K^*,\lambda)}_k\cap\C^{*(K^*,\lambda')}_{k+1}\right),\C^{*(K^*,\lambda')}_{k+2},\dots,\C^{*(K^*,\lambda')}_{K^*}\Biggr).
  \end{multline*}
  Conclude observing that $\Lambda_P(\widetilde{\C})=0$ i.e. $P\in\P_{\widetilde{\C}}$, which contradicts the definition of $K^*$.\\
  (ii). By Theorem \ref{thm:bucket_median}, any bucket order $\C$ such that $P\in\P_{\C}$ agrees with the Kemeny median.
  Then, observe that such bucket order $\C$ of size $K<K^*$ is obtained by iteratively merging adjacent buckets of $\C^{*(K^*,\lambda^*)}$:
  otherwise, following the proof of (i), we could define a new bucket order $\widetilde{\C}$ of size $K^*+1$ such that $P\in\P_{\widetilde{\C}}$.
  When $K=K^*-1$, Eq. (\ref{eq:merged_dim}) proves that $d_{\C}>d_{\C^{*(K^*,\lambda^*)}}$. The general result follows by induction.\\
	(iii). By induction on $n-K^*\in\{0,\dots,n-2\}$. Initialization is straightforward for $K^*=n$.
	Let $m\in\{3,\dots,n\}$ and assume that the proposition is true for any strongly/strictly stochastically transitive bucket distribution with $K^*=m$.
  Let $P$ be a strongly/strictly stochastically transitive bucket distribution with $K^*=m-1$.
  By definition of $K^*$, the algorithm runned with distribution $P$ cannot stop before computing $\C(n-m+1)$, which results from merging the adjacent buckets $\C_k(n-m)$ and $\C_{k+1}(n-m)$ (with $k\in\{1,\dots,m-1\}$).
	Then consider a distribution $\widetilde{P}$ with pairwise marginals $\widetilde{p}_{i,j}=1$ if $(i,j)\in\C_k(n-m)\times\C_{k+1}(n-m)$, $\widetilde{p}_{i,j}=0$ if $(i,j)\in\C_{k+1}(n-m)\times\C_k(n-m)$ and $\widetilde{p}_{i,j}=p_{i,j}$ otherwise.
  Hence, $\widetilde{P}$ is a strongly/strictly stochastically transitive bucket distribution and $\C(n-m)$ is, by construction of $\widetilde{P}$, returned by the algorithm when runned with distribution $\widetilde{P}$.
  Hence by induction hypothesis: $\widetilde{P}\in\P_{\C(n-m)}$.
  Conclude observing that $\Lambda_P(\C(n-m))=\Lambda_{\widetilde{P}}(\C(n-m))+\sum_{i\in\C_k(n-m), j\in\C_{k+1}(n-m)} p_{j, i} = \Delta_P^{(k)}(\C(n-m))$,
  which implies that $\Lambda_P(\C(n-m+1))=\Lambda_P(\C(n-m))-\Delta_P^{(k)}(\C(n-m))=0$.
\end{proof}

\section*{Appendix C - Alternative Cost Function: The Spearman $\rho$ Distance}

The expression of the distortion $\Lambda_P(\C)$ obtained in Proposition \ref{prop:kendall_prop}
critically depends on the choice of the Wasserstein parameters, namely $d=d_\tau$ the Kendall's $\tau$ distance and $q=1$.
Whereas, for general $d$ and $q$, obtaining a closed-analytical form for the distortion is a challenging problem, the following result shows that choosing the Spearman $\rho$ distance $d=d_2$ as cost function
and $q=2$ leads to an alternative distortion measure: $\Lambda'_{P}(\C)=\min_{P'\in \mathbf{P}_{\C}}W_{d_2,2}(P,P')$,
that can be explicitly expressed in terms of the triplet-wise probabilities $p_{i,j,k}=\mathbb{P}_{\Sigma\sim P}\{\Sigma(i)<\Sigma(j)<\Sigma(k)\}$. In addition, the coupling $(\Sigma,\Sigma_{\C})$ can also be shown to be optimal in this case: $\Lambda'_{P}(\C) = \mathbb{E}\left[d_2^2\left(\Sigma,\Sigma_{\C} \right)\right]$. Hence, based on the explicit formula below, the distortion can be straightforwardly estimated, just like the
$p_{i,j,k}$'s, so that an analysis similar to that in section \ref{sec:learning} in the Kendall's $\tau$ case, can be naturally carried out in order to provide statistical guarantees for the generalization capacity of empirical distortion minimization procedures.
\begin{proposition}
  \label{prop:spearman}
  Let $n\ge 3$ and $P$ be any distribution on $\Sn$. For any bucket order $\C=(\C_1,\; \ldots,\; \C_K)$, we have:
  \begin{equation*}
    \label{eq:Wspearman}
    \begin{split}
      \Lambda'_{P}(\C)
      &= \frac{2}{n-2} \sum_{1\le k<l<m\le K} \sum_{(a,b,c)\in\C_k\times\C_l\times\C_m} (n+1)p_{c,b,a}+n(p_{b,c,a}+p_{c,a,b})+p_{b,a,c}+p_{a,c,b}\\
      &+ \frac{2}{n-2} \sum_{1\le k<l\le K} \Biggl\{ \sum_{(a,b,c)\in\C_k\times\C_l\times\C_l} n(p_{b,c,a}+p_{c,b,a})+p_{b,a,c}+p_{c,a,b}\\
      &\qquad\qquad\qquad\qquad+\sum_{(a,b,c)\in\C_k\times\C_k\times\C_l} n(p_{c,a,b}+p_{c,b,a})+p_{a,c,b}+p_{b,c,a}\Biggr\}.
    \end{split}
  \end{equation*}
\end{proposition}

The proof is a straightforward consequence of the result stated below.

\begin{lemma}\label{lem:squared_spearman} Let $n\ge 3$ and $P$ be a probability distribution on $\mathfrak{S}_n$.
  \begin{itemize}
  \item[(i)] For any probability distribution $P'$ on $\mathfrak{S}_n$:
  $$
  W_{d_2,2}\left(P,P'  \right)
  \ge \frac{2}{n-2}\sum_{a<b<c}\left\{ \sum_{(i,j,k)\in\sigma(a,b,c)} \max(p_{i,j,k}, p'_{i,j,k}) - 1 \right\},
  $$
  where $\sigma(a,b,c)$ is the set of permutations of triplet $(a,b,c)$ and, for any $(i,j,k)\in\sigma(a,b,c)$, $p'_{i,j,k}=\mathbb{P}_{\Sigma\sim P'}\{\Sigma(i)<\Sigma(j)<\Sigma(k)\}$.
  \item[(ii)] If $P' \in \mathbf{P}_{\C}$ with $\C$ a bucket order of $\n$ with $K$ buckets:
  \begin{equation}
    \label{eq:Wspearman}
    \begin{split}
      &W_{d_2,2}\left(P,P'  \right) \ge\\
      &\frac{2}{n-2} \sum_{1\le k<l<m\le K} \sum_{(a,b,c)\in\C_k\times\C_l\times\C_m} (n+1)p_{c,b,a}+n(p_{b,c,a}+p_{c,a,b})+p_{b,a,c}+p_{a,c,b}\\
      &+ \frac{2}{n-2} \sum_{1\le k<l\le K} \Biggl\{ \sum_{(a,b,c)\in\C_k\times\C_l\times\C_l} n(p_{b,c,a}+p_{c,b,a})+p_{b,a,c}+p_{c,a,b}\\
      &\qquad\qquad\qquad\qquad\qquad\qquad+\sum_{(a,b,c)\in\C_k\times\C_k\times\C_l} n(p_{c,a,b}+p_{c,b,a})+p_{a,c,b}+p_{b,c,a}\Biggr\},
    \end{split}
  \end{equation}
  equality holding true when $P'=P_{\C}$, \textit{i.e.} when $P'$ is the distribution of $\Sigma_{\C}$.
\end{itemize}
\end{lemma}

\begin{proof} We start with proving the first assertion.

\noindent{\bf (i).}
Consider a coupling $(\Sigma,\Sigma')$ of two probability distributions $P$ and $P'$ on $\mathfrak{S}_n$.
Define the triplet-wise probabilities $p_{i,j,k}=\mathbb{P}_{\Sigma\sim P}\{\Sigma(i)<\Sigma(j)<\Sigma(k)\}$ and $p'_{i,j,k}=~\mathbb{P}_{\Sigma'\sim P'}\{\Sigma'(i)<\Sigma'(j)<\Sigma'(k)\}$.
For clarity's sake, we will assume that $\tilde{p}_{i,j,k} = \min(p_{i,j,k},p'_{i,j,k}) >0$ for all triplets $(i,j,k)$, the extension to the general case being straightforward.
We also denote $\bar{p}_{i,j,k} = \max(p_{i,j,k},p'_{i,j,k})$.
Given two pairs of three distinct elements of $\n$, $(i, j, k)$ and $(a, b, c)$, we define the following quantities:
$$
\pi_{a,b,c|i,j,k} = \mathbb{P}\left\{ \Sigma'(a)<\Sigma'(b)<\Sigma'(c)\mid \Sigma(i)<\Sigma(j)<\Sigma(k) \right\},
$$
$$
\pi'_{a,b,c|i,j,k} = \mathbb{P}\left\{ \Sigma(a)<\Sigma(b)<\Sigma(c)\mid \Sigma'(i)<\Sigma'(j)<\Sigma'(k) \right\},
$$
$$
\tilde{\pi}_{a,b,c|i,j,k} = \pi_{a,b,c|i,j,k}\mathbb{I}\{p_{i,j,k}\le p'_{i,j,k}\} + \pi'_{a,b,c|i,j,k}\mathbb{I}\{p_{i,j,k} > p'_{i,j,k}\},
$$
$$
\bar{\pi}_{a,b,c|i,j,k} = \pi_{a,b,c|i,j,k}\mathbb{I}\{p_{i,j,k} > p'_{i,j,k}\} + \pi'_{a,b,c|i,j,k}\mathbb{I}\{p_{i,j,k} \le p'_{i,j,k}\}.
$$
The motivation for defining the $\tilde{\pi}_{a,b,c|i,j,k}$'s is that the coupling condition $\tilde{\pi}_{i,j,k|i,j,k}=1$, which implies $\bar{\pi}_{i,j,k|i,j,k} = \frac{\tilde{p}_{i,j,k}}{\bar{p}_{i,j,k}}$,
is always feasible. By contrast, it necessarily holds that $\pi_{i,j,k|i,j,k}<1$ (resp. $\pi'_{i,j,k|i,j,k}<1$) when $p'_{i,j,k}<p_{i,j,k}$ (resp. $p_{i,j,k}<p'_{i,j,k}$).
Throughout the proof, the triplets $(a, b, c)$ always correspond to permutations of $(i, j, k)$.
Now write:
$$
\mathbb{E}\left[ d_2\left( \Sigma,\Sigma' \right)^2\right] = \sum_{i=1}^n \E[\Sigma(i)^2] + \E[\Sigma'(i)^2] - 2 \E[\Sigma(i) \Sigma'(i)]\\,
$$
where
$$
\E[\Sigma(i)^2] = \E[(1+\sum_{j\neq i}\mathbb{I}\{\Sigma(j)<\Sigma(i)\})^2] = 1 + \sum_{j\neq i} (n+1) p_{j,i} - \sum_{k\neq i,j} p_{j,i,k}
$$
and
\begin{multline*}
    \E[\Sigma(i) \Sigma'(i)] = 1 + \sum_{j\neq i} p_{j,i}+p'_{j,i}+\mathbb{P}\{\Sigma(j)<\Sigma(i),\Sigma'(j)<\Sigma'(i)\} \\
    + \sum_{k\neq i,j} \mathbb{P}\{\Sigma(j)<\Sigma(i),\Sigma'(k)<\Sigma'(i)\}.
\end{multline*}
Hence,
\begin{equation}
  \label{eq:spearman_triplets}
  \begin{split}
    \mathbb{E}\left[ d_2\left( \Sigma,\Sigma' \right)^2\right] =
    \sum_{a<b<c} \sum_{(i,j,k)\in \sigma(a,b,c)} & \frac{1}{n-2}\left\{(n-1) (p_{j,i}+p'_{j,i})-2\mathbb{P}\{\Sigma(j)<\Sigma(i),\Sigma'(j)<\Sigma'(i)\}\right\}\\
    &-p_{j,i,k}-p'_{j,i,k}-2\mathbb{P}\{\Sigma(j)<\Sigma(i),\Sigma'(k)<\Sigma'(i)\},
  \end{split}
\end{equation}
where $\sigma(a,b,c)$ is the set of the $6$ permutations of triplet $(a,b,c)$.
Some terms involved in Eq. (\ref{eq:spearman_triplets}) can be simplified when summing over $\sigma(a,b,c)$, namely:
$$
\sum_{(i,j,k)\in \sigma(a,b,c)} \frac{n-1}{n-2}(p_{j,i}+p'_{j,i})-p_{j,i,k}-p'_{j,i,k} = \frac{4n-2}{n-2}.
$$
We now simply have:
\begin{equation}
  \begin{split}
    \mathbb{E}\left[ d_2\left( \Sigma,\Sigma' \right)^2\right] =
    \sum_{a<b<c} \frac{4n-2}{n-2} - 2\sum_{(i,j,k)\in \sigma(a,b,c)} & \frac{1}{n-2}\mathbb{P}\{\Sigma(j)<\Sigma(i),\Sigma'(j)<\Sigma'(i)\}\\
    &+\mathbb{P}\{\Sigma(j)<\Sigma(i),\Sigma'(k)<\Sigma'(i)\}.
  \end{split}
\end{equation}
Observe that for all triplets $(a,b,c)$ and $(i,j,k)$:
\begin{multline*}
\mathbb{P}(\Sigma'(a)<\Sigma'(b)<\Sigma'(c), \Sigma(i)<\Sigma(j)<\Sigma(k))
    \\+ \mathbb{P}(\Sigma'(i)<\Sigma'(j)<\Sigma'(k), \Sigma(a)<\Sigma(b)<\Sigma(c))
    = \pi_{a,b,c|i,j,k}p_{i,j,k} + \pi'_{a,b,c|i,j,k}p'_{i,j,k}.
\end{multline*}
Then, by the law of total probability, we have for all distinct $i,\; j,\; k$,
\begin{multline*}
    \mathbb{P}\{\Sigma(j)<\Sigma(i),\Sigma'(j)<\Sigma'(i)\}
    = \frac{1}{2}\{\pi_{j,k,i|j,k,i}p_{j,k,i}+\pi'_{j,k,i|j,k,i}p'_{j,k,i}\}\\
    + \frac{1}{2}\{\pi_{k,j,i|k,j,i}p_{k,j,i}+\pi'_{k,j,i|k,j,i}p'_{k,j,i}\}
    + \frac{1}{2}\{\pi_{j,i,k|j,i,k}p_{j,i,k}+\pi'_{j,i,k|j,i,k}p'_{j,i,k}\}\\
    + \frac{1}{2}\{\pi_{j,i,k|j,k,i}p_{j,k,i}+\pi'_{j,i,k|j,k,i}p'_{j,k,i}+\pi_{j,k,i|j,i,k}p_{j,i,k}+\pi'_{j,k,i|j,i,k}p'_{j,i,k}\}\\
    + \frac{1}{2}\{\pi_{k,j,i|j,k,i}p_{j,k,i}+\pi'_{k,j,i|j,k,i}p'_{j,k,i}+\pi_{j,k,i|k,j,i}p_{k,j,i}+\pi'_{j,k,i|k,j,i}p'_{k,j,i}\}\\
    + \frac{1}{2}\{\pi_{j,i,k|k,j,i}p_{k,j,i}+\pi'_{j,i,k|k,j,i}p'_{k,j,i} + \pi_{k,j,i|j,i,k}p_{j,i,k}+\pi'_{k,j,i|j,i,k}p'_{j,i,k}\},
\end{multline*}
and
\begin{multline*}
    \mathbb{P}\{\Sigma(j)<\Sigma(i),\Sigma'(k)<\Sigma'(i)\}
    = \frac{1}{2}\{\pi_{j,k,i|j,k,i}p_{j,k,i} + \pi'_{j,k,i|j,k,i}p'_{j,k,i}\}\\
    + \frac{1}{2}\{\pi_{k,j,i|k,j,i} p_{k,j,i} + \pi'_{k,j,i|k,j,i} p'_{k,j,i}\}\\
    + \frac{1}{2}\{\pi_{k,j,i|j,k,i}p_{j,k,i} + \pi'_{k,j,i|j,k,i}p'_{j,k,i} + \pi_{j,k,i|k,j,i}p_{k,j,i} + \pi'_{j,k,i|k,j,i}p'_{k,j,i}\}\\
    + \mathbb{P}\{\Sigma'(j)<\Sigma'(k)<\Sigma'(i), \Sigma(j)<\Sigma(i)<\Sigma(k)\}\\
    + \mathbb{P}\{\Sigma'(k)<\Sigma'(i)<\Sigma'(j), \Sigma(j)<\Sigma(k)<\Sigma(i)\}\\
    + \mathbb{P}\{\Sigma'(k)<\Sigma'(j)<\Sigma'(i), \Sigma(j)<\Sigma(i)<\Sigma(k)\}\\
    + \mathbb{P}\{\Sigma'(k)<\Sigma'(i)<\Sigma'(j), \Sigma(k)<\Sigma(j)<\Sigma(i)\}\\
    + \mathbb{P}\{\Sigma'(k)<\Sigma'(i)<\Sigma'(j), \Sigma(j)<\Sigma(i)<\Sigma(k)\},
\end{multline*}
which implies:
\begin{multline*}
  \mathbb{P}\{\Sigma(j)<\Sigma(i),\Sigma'(k)<\Sigma'(i)\} + \mathbb{P}\{\Sigma(k)<\Sigma(i),\Sigma'(j)<\Sigma'(i)\}\\
  = \pi_{j,k,i|j,k,i}p_{j,k,i} + \pi'_{j,k,i|j,k,i}p'_{j,k,i}
  + \pi_{k,j,i|k,j,i} p_{k,j,i} + \pi'_{k,j,i|k,j,i} p'_{k,j,i}\\
  + \pi_{k,j,i|j,k,i}p_{j,k,i} + \pi'_{k,j,i|j,k,i}p'_{j,k,i} + \pi_{j,k,i|k,j,i}p_{k,j,i} + \pi'_{j,k,i|k,j,i}p'_{k,j,i}\\
  + \frac{1}{2}\left\{\pi_{j,k,i|j,i,k}p_{j,i,k} + \pi'_{j,k,i|j,i,k}p'_{j,i,k} + \pi_{j,i,k|j,k,i}p_{j,k,i} + \pi'_{j,i,k|j,k,i}p'_{j,k,i}\right\}\\
  + \frac{1}{2}\left\{\pi_{k,i,j|j,k,i}p_{j,k,i} + \pi'_{k,i,j|j,k,i}p'_{j,k,i} + \pi_{j,k,i|k,i,j}p_{k,i,j} + \pi'_{j,k,i|k,i,j}p'_{k,i,j}\right\}\\
  + \frac{1}{2}\left\{\pi_{k,j,i|j,i,k}p_{j,i,k} + \pi'_{k,j,i|j,i,k}p'_{j,i,k} + \pi_{j,i,k|k,j,i}p_{k,j,i} + \pi'_{j,i,k|k,j,i}p'_{k,j,i}\right\}\\
  + \frac{1}{2}\left\{\pi_{k,i,j|k,j,i}p_{k,j,i} + \pi'_{k,i,j|k,j,i}p'_{k,j,i} + \pi_{k,j,i|k,i,j}p_{k,i,j} + \pi'_{k,j,i|k,i,j}p'_{k,i,j}\right\}\\
  + \frac{1}{2}\left\{\pi_{k,i,j|j,i,k}p_{j,i,k} + \pi'_{k,i,j|j,i,k}p'_{j,i,k} + \pi_{j,i,k|k,i,j}p_{k,i,j} + \pi'_{j,i,k|k,i,j}p'_{k,i,j}\right\},
\end{multline*}
which is invariant under permutation of the indices $j$ and $k$.
Hence,
\begin{equation}
  \label{eq:Habc}
  \begin{split}
    &H(a,b,c) = \sum_{(i,j,k)\in \sigma(a,b,c)} \frac{1}{n-2}\mathbb{P}\{\Sigma(j)<\Sigma(i),\Sigma'(j)<\Sigma'(i)\}
    +\mathbb{P}\{\Sigma(j)<\Sigma(i),\Sigma'(k)<\Sigma'(i)\}\\
    &= \sum_{(i,j,k)\in \sigma(a,b,c)} \Biggl\{\frac{2n-1}{2(n-2)}\tilde{\pi}_{j,k,i|j,k,i} + \frac{n-1}{n-2}(\tilde{\pi}_{k,j,i|j,k,i}+ \tilde{\pi}_{j,i,k|j,k,i})\\
    &\qquad\qquad\qquad\qquad + \frac{n-1}{2(n-2)}( \tilde{\pi}_{k,i,j|j,k,i} + \tilde{\pi}_{i,j,k|j,k,i} ) + \frac{1}{2}\tilde{\pi}_{i,k,j|j,k,i} \Biggr\}\tilde{p}_{j,k,i}\\
    &+ \Biggl\{\frac{2n-1}{2(n-2)}\bar{\pi}_{j,k,i|j,k,i} + \frac{n-1}{n-2}(\bar{\pi}_{k,j,i|j,k,i} + \bar{\pi}_{j,i,k|j,k,i}) + \frac{n-1}{2(n-2)}( \bar{\pi}_{k,i,j|j,k,i} + \bar{\pi}_{i,j,k|j,k,i} ) + \frac{1}{2}\bar{\pi}_{i,k,j|j,k,i} \Biggr\}\bar{p}_{j,k,i},
  \end{split}
\end{equation}
which is maximum when $\tilde{\pi}_{j,k,i|j,k,i}=1$ (which implies $\bar{\pi}_{j,k,i|j,k,i}=\frac{\tilde{p}_{j,k,i}}{\bar{p}_{j,k,i}}$) and $\bar{\pi}_{k,j,i|j,k,i}+\bar{\pi}_{j,i,k|j,k,i} = 1 - \frac{\tilde{p}_{j,k,i}}{\bar{p}_{j,k,i}}$
for all $(i,j,k)\in\sigma(a,b,c)$ and then verifies:
\begin{equation}
\label{eq:Habc_max}
\begin{split}
  H(a,b,c) &\le \sum_{(i,j,k)\in \sigma(a,b,c)} \frac{n}{n-2}\tilde{p}_{i,j,k} + \frac{n-1}{n-2}\bar{p}_{i,j,k}
  = \frac{1}{n-2}\sum_{(i,j,k)\in \sigma(a,b,c)} n(p_{i,j,k}+p'_{i,j,k}) - \bar{p}_{i,j,k}\\
  &= \frac{1}{n-2}\left\{ 2n - \sum_{(i,j,k)\in \sigma(a,b,c)} \bar{p}_{i,j,k}\right\},
\end{split}
\end{equation}
which concludes the first part of the proof.

\noindent{\bf (ii).}
Now consider the particular case $P'\in \mathbf{P}_{\C}$, with $\C$ a bucket order of $\n$ with $K$ buckets.
We propose to prove that $\min_{P'\in\mathbf{P}_{\C}} W_{d_2,2}(P,P')= W_{d_2,2}(P,P_{\C}) =\E[d_2^2(\Sigma,\Sigma_{\C})]$ and to obtain an explicit expression.
Given three distinct indices $(a, b, c)\in\n^3$, we consider the following four possible scenarios.\\

\noindent {\bf Case $1$: $(a,b,c)\in \C_q^3$ are in the same bucket.} The maximizing conditions for $H(a,b,c)$ in Eq. (\ref{eq:Habc}) are $\tilde{\pi}_{j,k,i|j,k,i}=1$ and $\bar{\pi}_{k,j,i|j,k,i}+\bar{\pi}_{j,i,k|j,k,i} = 1 - \frac{\tilde{p}_{j,k,i}}{\bar{p}_{j,k,i}}$ for all $(i,j,k)\in\sigma(a,b,c)$.
All are verified when $\Sigma'=\Sigma_{\C}$ as $\Sigma(i)<\Sigma(j)<\Sigma(k)$ iff $\Sigma_{\C}(i)<\Sigma_{\C}(j)<\Sigma_{\C}(k)$.
Hence:
$$
H(a,b,c) \le \frac{2n-1}{n-2},
$$
with equality when $\Sigma'=\Sigma_{\C}$.\\

\noindent {\bf Case $2$: $(a,b,c)\in \C_q\times\C_r\times\C_s$ are in three different buckets (e.g. $q<r<s$).}
For all $(j,k,i)\in\sigma(a,b,c)\setminus\{(a,b,c)\}$, $p'_{j,k,i}=\tilde{p}_{j,k,i}=0$.
Hence, $H(a,b,c)$ writes without the terms related to the five impossible events $\Sigma'(j)<\Sigma'(k)<\Sigma'(i)$.
Moreover, $\bar{p}_{j,k,i}=p_{j,k,i}$ and $\bar{\pi}_{a,b,c|j,k,i}=1$ so the sum of the corresponding contributions in $H(a,b,c)$ is:
\begin{equation}
\label{eq:3diff_buckets_1}
  \frac{n-1}{n-2}(p_{b,a,c}+p_{a,c,b}) + \frac{n-1}{2(n-2)}(p_{b,c,a}+p_{c,a,b}) + \frac{1}{2}p_{c,b,a}.
\end{equation}
We have $p_{a,b,c}\le p'_{a,b,c}=1$ so $\tilde{\pi}_{a,b,c|a,b,c}=1$
and for all $(i,j,k)\in\sigma(a,b,c)$, $\bar{\pi}_{i,j,k|a,b,c}=p_{i,j,k}$. The sum of the corresponding contributions in $H(a,b,c)$ is:
\begin{equation}
  \label{eq:3diff_buckets_2}
  \frac{2n-1}{n-2}p_{a,b,c} + \frac{n-1}{n-2}(p_{b,a,c}+p_{a,c,b}) + \frac{n-1}{2(n-2)}(p_{b,c,a}+p_{c,a,b}) + \frac{1}{2}p_{c,b,a}.
\end{equation}
Finally, by summing expressions (\ref{eq:3diff_buckets_1}) and (\ref{eq:3diff_buckets_2}),
$$
H(a,b,c) = \frac{2n-1}{n-2}p_{a,b,c} + \frac{2(n-1)}{n-2}(p_{b,a,c}+p_{a,c,b}) + \frac{n-1}{n-2}(p_{b,c,a}+p_{c,a,b}) + p_{c,b,a}.
$$

\noindent {\bf Case $3$: $(a,b,c)\in \C_q\times\C_r\times\C_r$ are in two different buckets such that one item (here $a$) is ranked first among the triplet (i.e. $q<r$).}
For all $(j,k,i)\in\sigma(a,b,c)\setminus\{(a,b,c),(a,c,b)\}$, $p'_{j,k,i}=\tilde{p}_{j,k,i}=0$.
Hence, $H(a,b,c)$ writes without the terms related to the four impossible events $\Sigma'(j)<\Sigma'(k)<\Sigma'(i)$.
For all $(j,k,i)\in\sigma(a,b,c)$, $\pi_{a,b,c|j,k,i}+\pi_{a,c,b|j,k,i}=1$ and the sum of the corresponding contributions in $H(a,b,c)$ is:
\begin{equation}
  \label{eq:2diff_buckets_0}
  \begin{split}
    &\left(\frac{2n-1}{2(n-2)}\pi_{a,b,c|a,b,c} + \frac{n-1}{n-2}\pi_{a,c,b|a,b,c}\right)p_{a,b,c}
    +\left(\frac{2n-1}{2(n-2)}\pi_{a,c,b|a,c,b} + \frac{n-1}{n-2}\pi_{a,b,c|a,c,b}\right)p_{a,c,b}\\
    &+\left(\frac{n-1}{2(n-2)}\pi_{a,b,c|b,c,a} + \frac{1}{2}\pi_{a,c,b|b,c,a}\right)p_{b,c,a}
    +\left(\frac{n-1}{n-2}\pi_{a,b,c|b,a,c} + \frac{n-1}{2(n-2)}\pi_{a,c,b|b,a,c}\right)p_{b,a,c}\\
    &+\left(\frac{n-1}{2(n-2)}\pi_{a,c,b|c,b,a} + \frac{1}{2}\pi_{a,b,c|c,b,a}\right)p_{c,b,a}
    +\left(\frac{n-1}{n-2}\pi_{a,c,b|c,a,b} + \frac{n-1}{2(n-2)}\pi_{a,b,c|c,a,b}\right)p_{c,a,b}.
  \end{split}
\end{equation}
Observe that the expression above is maximum when $\pi_{a,b,c|a,b,c}=\pi_{a,c,b|a,c,b}=\pi_{a,b,c|b,c,a}=\pi_{a,b,c|b,a,c}=\pi_{a,c,b|c,b,a}=\pi_{a,c,b|c,a,b}=1$,
which is verified if $\Sigma'=\Sigma_{\C}$. In this case, (\ref{eq:2diff_buckets_0}) writes:
\begin{equation}
\label{eq:2diff_buckets_1}
    \frac{2n-1}{2(n-2)}(p_{a,b,c}+p_{a,c,b})+\frac{n-1}{n-2}(p_{b,a,c}+p_{c,a,b}) + \frac{n-1}{2(n-2)}(p_{b,c,a}+p_{c,b,a}).
\end{equation}
Now consider $(j,k,i)\in\{(a,b,c),(a,c,b)\}$: $p'_{a,b,c} + p'_{a,c,b} = 1$ and the corresponding contributions in $H(a,b,c)$ sum as follows:
\begin{equation*}
  \begin{split}
    &\Biggl\{\frac{2n-1}{2(n-2)}\pi'_{a,b,c|a,b,c} + \frac{n-1}{n-2}(\pi'_{b,a,c|a,b,c}+\pi'_{a,c,b|a,b,c})\\
    &+ \frac{n-1}{2(n-2)}(\pi'_{b,c,a|a,b,c}+\pi'_{c,a,b|a,b,c}) + \frac{1}{2}\pi'_{c,b,a|a,b,c}\Biggr\}p'_{a,b,c}\\
    & +\Biggl\{\frac{2n-1}{2(n-2)}\pi'_{a,c,b|a,c,b} + \frac{n-1}{n-2}(\pi'_{c,a,b|a,c,b}+\pi'_{a,b,c|a,c,b})\\
    &+ \frac{n-1}{2(n-2)}(\pi'_{c,b,a|a,c,b}+\pi'_{b,a,c|a,c,b}) + \frac{1}{2}\pi'_{b,c,a|a,c,b}\Biggr\}p'_{a,c,b},
  \end{split}
\end{equation*}
which is maximum when $\pi'_{a,c,b|a,b,c}=\pi'_{c,a,b|a,b,c}=\pi'_{c,b,a|a,b,c}=0$
and $\pi'_{a,b,c|a,c,b}=\pi'_{b,a,c|a,c,b}=\pi'_{b,c,a|a,c,b}=0$: both conditions are verified for $\Sigma'=\Sigma_{\C}$.
Then, the expression above is upper bounded by:
\begin{equation}
  \label{eq:2diff_buckets_2}
  \frac{2n-1}{2(n-2)}(p_{a,b,c}+p_{a,c,b})+\frac{n-1}{n-2}(p_{b,a,c}+p_{c,a,b}) + \frac{n-1}{2(n-2)}(p_{b,c,a}+p_{c,b,a}),
\end{equation}
with equality when $\Sigma'=\Sigma_{\C}$.
Finally, by summing (\ref{eq:2diff_buckets_1}) and (\ref{eq:2diff_buckets_2}),
$$
H(a,b,c) \le \frac{2n-1}{n-2}(p_{a,b,c}+p_{a,c,b}) + \frac{2(n-1)}{n-2}(p_{b,a,c}+p_{c,a,b}) + \frac{n-1}{n-2}(p_{b,c,a}+p_{c,b,a}),
$$
with equality when $\Sigma'=\Sigma_{\C}$.\\

\noindent {\bf Case $4$: $(a,b,c)\in \C_q\times\C_q\times\C_r$ are in two different buckets such that one item (here $c$) is ranked last among the triplet (i.e. $q<r$).}
By symmetry with the previous situation, we obtain:
$$
H(a,b,c) \le \frac{2n-1}{n-2}(p_{a,b,c}+p_{b,a,c}) + \frac{2(n-1)}{n-2}(p_{a,c,b}+p_{b,c,a}) + \frac{n-1}{n-2}(p_{c,a,b}+p_{c,b,a}),
$$
with equality when $\Sigma'=\Sigma_{\C}$.

\end{proof}

\section*{Appendix D - Experiments on toy datasets}

\begin{figure}[h!]
	\centering
	\includegraphics[width=.325\textwidth]{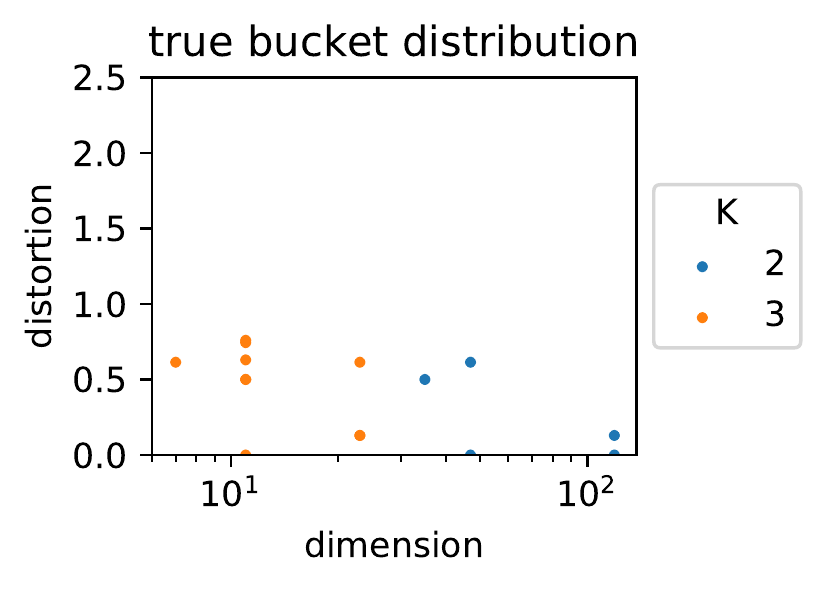}
	\includegraphics[width=.325\textwidth]{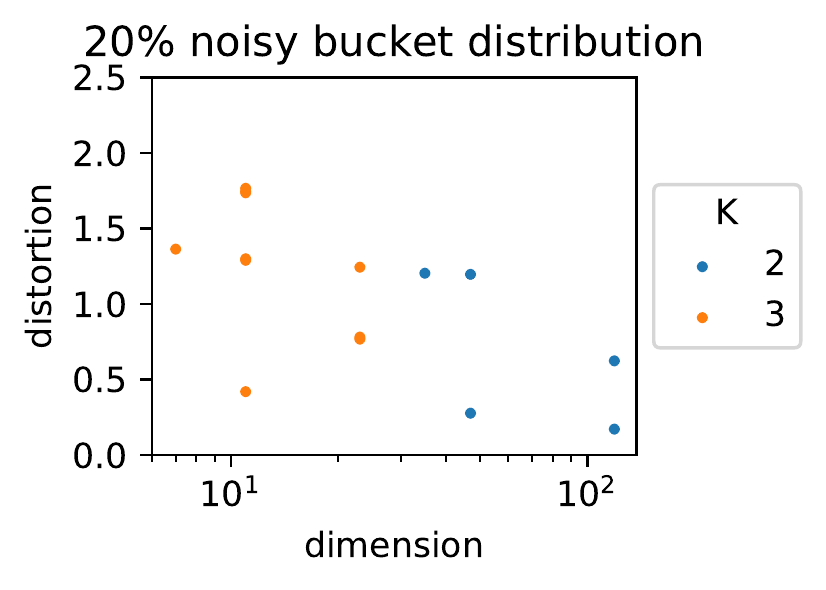}\hfill
	\includegraphics[width=.335\textwidth]{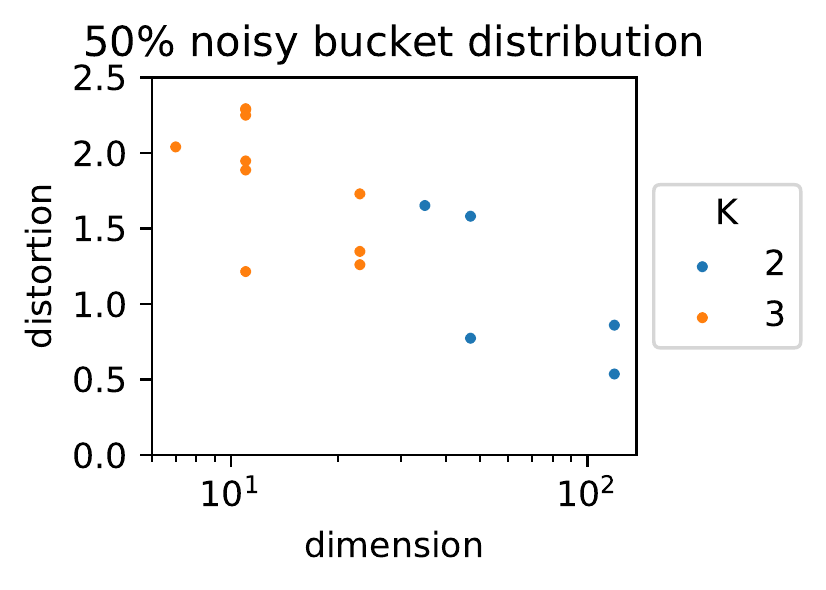}\hfill
	\caption{Dimension-Distortion plot for different bucket sizes on simulated datasets.}
	\label{fig:simulated_data}
\end{figure}

\noindent We now provide an illustration of the notions we introduced in this paper, in particular of a bucket distribution and of our distortion criteria. For $n=6$ items, we fixed a bucket order $\C=(\C_1, \C_2, \C_3)$ of shape $\lambda=(2,3,1)$ and considered a bucket distribution $P\in\P_{\C}$. Specifically, $P$ is the uniform distribution over all the permutations extending the bucket order $\C$ and has thus its pairwise marginals such that $p_{j,i}=0$ as soon as $(i,j)\in \C_k \times \C_l$ with $k<l$. In Figure \ref{fig:simulated_data}, the first plot on the left is a scatterplot of all buckets of size $K\in\{2,3\}$ where for any bucket $\C'$ of size $K$, the horizontal axis is the distortion $\Lambda_{P}(\C')$ (see \eqref{eq:mt_criterion}) and the vertical axis is the dimension of $\P_{\C'}$ in log scale. On the left plot, one can see that one bucket of size $K=3$ attains a null distortion, i.e. when $\C'=\C$, and two buckets of size $K=2$ as well, i.e. when $\C'=(\C_1 \cup \C_2, \C_3)$ and when $\C'=(\C_1, \C_2 \cup \C_3)$. Then, a dataset of 2000 samples from $P$ was drawn, and for a certain part of the samples, a pair of items was randomly swapped within the sample. The middle and right plot thus represent the empirical distortions $\widehat{\Lambda}_N(\C')$ for any $\C'$ computed on these datasets, where respectively 20\% and 50\% of the samples were contaminated. One can notice that the distortion is increasing with the noise, still, the best bucket of size $3$ remains $\C'=\C$. However, the buckets $\C'$ attaining the minimum distortion in the noisy case are of size $2$, because the distortion involves a smaller number of terms $\kappa(\lambda_{\C'})$ for a smaller size.

\begin{figure}[h!]
	\centering
	\begin{tabular}{cc}
		\includegraphics[width=.325\textwidth]{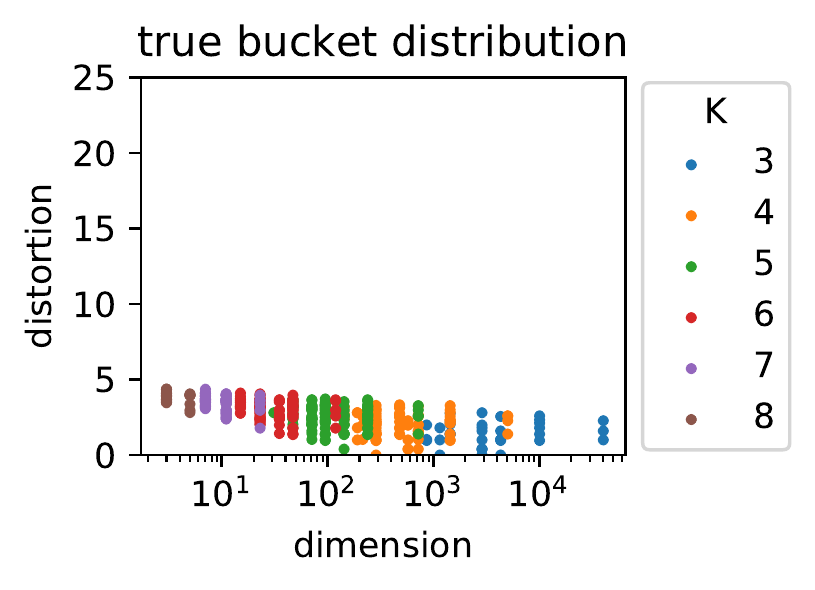} &
		\includegraphics[width=.325\textwidth]{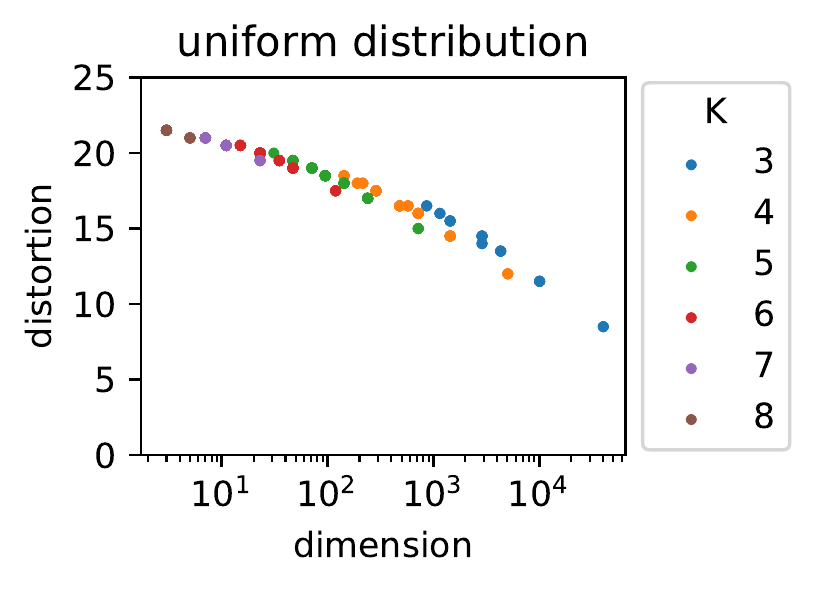}\hfill\\
		\includegraphics[width=.325\textwidth]{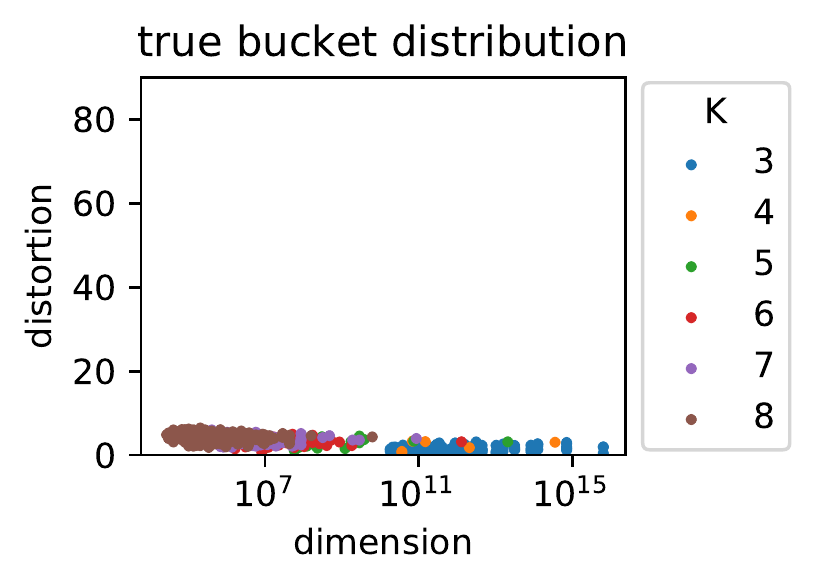}&
		\includegraphics[width=.325\textwidth]{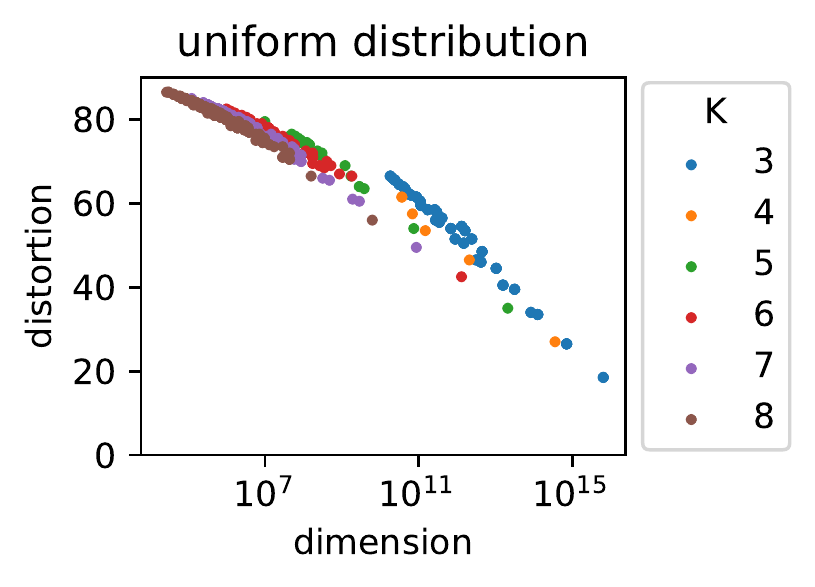}\hfill
	\end{tabular}
	\caption{Dimension-Distortion plot for a true bucket distribution versus a uniform distribution ($n=10$ on top and $n=20$ below).}
		\label{fig:bucket_vs_uniform}
\end{figure}

\noindent We now perform a second experiment. We want to compare the distortion versus dimension graph for a true bucket distribution (i.e., for a collection of pairwise marginals that respect a bucket order) and for a uniform distribution (i.e., a collection of pairwise marginals where $p_{j,i}=0.5$ for all $i,j$). This corresponds to the plots on Figure~\ref{fig:bucket_vs_uniform}. One can notice that the points are much more spread for a true bucket distribution, since some buckets will attain a very low distortion (those who agree with the true one) while some have a high distortion. In contrast, for a uniform distribution, all the buckets will perform relatively in the same way, and the scatter plot is much more compact.

\end{document}